%% file: main.tex
\documentclass[11pt, a4paper, oneside,reqno]{amsart}

\input{arxiv/AMIN_style.tex}

\usepackage[square,numbers]{natbib}

\usepackage[utf8]{inputenc}
\usepackage[T1]{fontenc}
\usepackage{microtype}
\usepackage{courier}

\usepackage{dsfont,amsfonts,amssymb,amsmath,amsthm,bm,mathtools,mathrsfs}
\usepackage{nicefrac}

\usepackage{booktabs}
\usepackage{multirow}
\usepackage{graphicx}
\usepackage[font=small,labelfont=rm]{subcaption}
\usepackage[font=small,margin=10pt]{caption}
\usepackage{wrapfig}

\usepackage{algorithm,algorithmic}

\usepackage{hyperref}
\usepackage{url}
\usepackage{xcolor}
\usepackage{enumitem}
\usepackage{tikz}
\usetikzlibrary{arrows.meta, positioning, calc, backgrounds, decorations.pathreplacing}

\input{arxiv/notation.tex}

\graphicspath{{figure/}}

\newcommand{\Cn}{\op{C}_{n, \lambda}^{\mu}}
\newcommand{\N}{\mathbb{N}}

\newcommand{\HI}{\textbf{HI}}
\newcommand{\HIpi}{\textbf{HI($\bm{\pi}$)}}
\newcommand{\LKHI}{\textbf{LKHI}}
\newcommand{\LKHIpi}{\textbf{LKHI($\bm{\pi}$)}}

\title[]{Limiting-Kernel Q($\lambda$): Bridging Short and Long Horizons}
\date{}

\author[]{Tolga~Ok$^{1}$, Arman~Sharifi~Kolarijani$^{1,2}$, Peyman~{Mohajerin~Esfahani}$^{1,3}$, and Mohamad Amin Sharifi Kolarijani$^1$ \\
\\
$^1$Delft University of Technology, The Netherlands\\
$^2$Alpha Brain Technologies, The Netherlands \\
$^3$University of Toronto, Canada
}
\thanks{Correspondence to: Tolga Ok $<$\texttt{t.ok@tudelft.nl}$>$.}
\thanks{This work was partially supported by the European Research Council (ERC) project TRUST-949796, the Horizon Europe Pathfinder Open project RELIEVE-101099481, and the NSERC Discovery grant RGPIN-2025-06544.}  

\begin{document}

\begin{abstract}
In reinforcement learning, more accurate policy evaluation has been frequently shown to improve downstream policy optimization performance.
For tractability, most value estimators rely on $n$-step truncation and thereby omit the long-horizon behavior of the transition dynamics.
Methods that instead capture this behavior through the global structure of the dynamics can converge faster, but their computational and memory requirements limit their scalability to large or continuous state spaces.
To reconcile these limitations, we introduce Limiting-Kernel Q($\lambda$) (LKQL), an off-policy value estimator that uses the limiting kernel to approximate the long-horizon terms omitted by truncation.
LKQL has the same order of complexity as $n$-step estimators and integrates directly into both on- and off-policy actor-critic algorithms.
We prove that, under aperiodicity and in the near-on-policy regime, the operator underlying LKQL improves the policy evaluation convergence rate over its truncated counterpart for sufficiently large $n$, and that LKQL itself converges almost surely to the optimal values in finite Markov decision processes (MDPs) under a fixed behavior policy.
On the MuJoCo continuous-control benchmark, we show that LKQL improves over $n$-step baselines in most settings, particularly on long-horizon tasks.

\end{abstract}

\maketitle

\section{Introduction}\label{sec:intro}

Policy evaluation underlies much of value-based reinforcement learning.
Given a Markov decision process (MDP), it entails finding the value function $Q^\pi$ of a policy $\pi$, typically by solving the Bellman equation $Q = \op{T}^\pi Q$, where $\op{T}^\pi$ denotes the Bellman operator.
In finite MDPs, the convergence rate of the overall algorithm is governed by how accurately each evaluation step solves the Bellman equation~\cite{puterman1994markov, Scherrer2016improved}.
In MDPs with large or continuous state spaces, policy evaluation instead determines the target values to which a parameterized value function is regressed.
The suboptimality of the resulting greedy policy is in turn bounded by the accuracy of these targets~\cite{farahmand2010error, scherrer2015approximate}.
Empirically, more accurate targets are consistently found among the most effective design choices~\cite{hessel2018rainbow, fedus2020revisiting, hernandez2019understanding}.
Together, these results motivate more accurate approximations of the full evaluation step, provided they remain computationally tractable and estimable from \emph{rollouts}, i.e., trajectories sampled from the MDP.

Existing approximations of the evaluation step build on the structure of its exact solution, which admits the closed form
\begin{equation}
	\label{eq:neumann}
	Q^\pi = (\op{I} - \gamma\op{P}^{\pi})^{-1} r = \ssum_{k \geq 0} (\gamma\op{P}^{\pi})^{k}\, r,
\end{equation}
where $\gamma \in (0,1)$ is the discount factor, $\op{P}^{\pi}$ is the transition operator of the chain induced by $\pi$, and $r$ is the reward function.
We refer to the inverse in~\eqref{eq:neumann} as the \emph{preconditioner}.
Applying this preconditioner to obtain $Q^\pi$, however, amounts to solving a linear system whose dimension is the number of state-action pairs.

One common approximation therefore truncates the series in~\eqref{eq:neumann} after $n$ terms.
This is an appealing choice since the truncated sum admits a computationally inexpensive estimator from a single rollout, at no cost beyond computing its temporal-difference (TD) errors.
Estimators of this form, known as \emph{return-based} or $\lambda$-return estimators, underpin a wide range of deep reinforcement learning algorithms.
This includes the $n$-step returns of A3C~\cite{mnih2016asynchronous}, Rainbow~\cite{hessel2018rainbow}, and R2D2~\cite{kapturowski2019r2d2}; the generalized advantage estimation (GAE) of PPO~\cite{schulman2016gae, schulman2017proximal}; the V-trace targets of IMPALA~\cite{espeholt2018impala}; and the off-policy returns including Retrace~\cite{munos2016safe}, Tree-Backup~\cite{precup2000eligibility}, Peng's Q($\lambda$)~\cite{peng1994incremental, kozuno2021revisiting}, and Harutyunyan's Q($\lambda$)~\cite{harutyunyan2016q}.
Moreover, these estimators update the values of only the sampled state-action pairs at each iteration.
This locality renders them readily compatible with the stochastic-gradient training of parameterized value functions.

A second line of work instead approximates the preconditioner as a whole, leveraging the global structure of the underlying chain.
This structure appears as the leading eigenvectors (eigenfunctions) of the transition operator in rank-one corrections~\cite{bertsekas1995generic} and Deflated Dynamics Value Iteration (DDVI)~\cite{lee2024deflated}, an estimate of a stationary distribution in rank-one Q-learning (R1-QL)~\cite{kolarijani2025rank}, and a direct estimate of the exact preconditioner in Zap Q-learning~\cite{devraj2017zap, chen2020zap}.
Although their mechanisms and assumptions differ, these methods generally do not scale to MDPs with large or continuous state spaces, as they either require global updates over the entire state-action space at each iteration or rely on matrix operations that scale superlinearly with the number of parameters.

In this paper, we propose a \emph{tractable} spectral approximation of the preconditioner, that is,
\begin{itemize}
    \item[(i)] \emph{estimable} from sampled rollouts, and
    \item[(ii)] \emph{scalable} to large or continuous state spaces, with the same order of per-sample (i.e., rollout) complexity as return-based estimators.
\end{itemize}

Our approximation builds on the long-horizon behavior of the transition dynamics.
For a chain whose recurrent classes are aperiodic, the powers $(\op{P}^{\pi})^k$ converge geometrically to the \emph{limiting kernel} (LK), $\lim_{k\to\infty}(\op{P}^{\pi})^k$~\cite[\S4.4.2]{gallager2013stochastic}.
The LK thus provides an increasingly accurate approximation of these powers in the \emph{tail} of the series in~\eqref{eq:neumann}, particularly for terms with large $k$.
This suggests a natural extension of $n$-step truncation, which we call \emph{tail completion}, that retains the first $n$ terms in~\eqref{eq:neumann} while approximating the truncated terms using the LK.

Based on \emph{tail completion}, we develop a value estimator for \emph{off-policy} learning, where the objective is to estimate the value function $Q^\pi$ of a target policy $\pi$ from rollouts generated by a behavior policy $\mu$.
Specifically, we build upon Harutyunyan's Q($\lambda$) (HQL)~\cite{harutyunyan2016q}, a non-conservative off-policy operator commonly implemented through $n$-step truncation.
For HQL, the tail approximation reduces to a single correction, which we call the \emph{LK term}.
The LK term tracks the long-run average of the TD errors and can be estimated alongside the value function, preserving the per-sample complexity and locality of return-based estimators.

The resulting estimator admits a realization entirely by TD learning, coupling the standard backup for the value functions with a second backup for the LK term.
We term this realization the \emph{double backup}, from which we derive our value learning algorithm, \emph{Limiting-Kernel Q($\lambda$)} (LKQL).
In practice, each backup maintains its own function approximation, one for the values and one for the LK term, allowing LKQL to integrate directly into on- and off-policy actor-critic frameworks.

In summary, our contributions are as follows.
\begin{itemize}
	\item[(i)] \textbf{Analysis of the truncated operator.} We derive explicit convergence rates for the $n$-step Harutyunyan operator in both policy evaluation and control (Lemma~\ref{lem:convergence-h}), characterizing how the truncation length $n$ governs both the rate and region of convergence.
	\item[(ii)] \textbf{A limiting-kernel operator with provable rate improvement.} We introduce the LK preconditioner and establish convergence guarantees for the resulting LK Harutyunyan operator, characterizing the regime in which it improves on the rate of its truncated counterpart (Theorem~\ref{thm:conv-lkhi-pi}, Corollary~\ref{cor:conv-lkhi-pi}, and Theorem~\ref{thm:convergence-control}).
	\item[(iii)] \textbf{A scalable long-horizon extension of $n$-step estimators.} We introduce LKQL and prove that it converges almost surely to the optimal values under a fixed behavior policy and i.i.d.\ initial state-action pairs (Theorem~\ref{thm:lkql-value-learn}).
	On the MuJoCo continuous-control benchmark, we demonstrate that LKQL improves over the baseline value estimators (namely, GAE~\cite{schulman2016gae} in on-policy learning and truncated HQL (THQL) in off-policy learning), particularly on long-horizon tasks (Section~\ref{sec:evaluation}).
\end{itemize}
The remainder of the paper develops these contributions in turn.
Section~\ref{sec:preliminaries} reviews the operator preliminaries and analyzes the truncated operator. Section~\ref{sec:acc-ops} develops the LK operators and LKQL. Section~\ref{sec:evaluation} introduces the actor-critic integration of LKQL and reports the empirical results.
Section~\ref{sec:conc} concludes the paper by discussing the limitations and possible extensions of the current work.
All the technical proofs are provided in Appendix~\ref{appx:proofs}. 

\textbf{Notation.}
We write $\N \Let \{1, 2, \ldots\}$ and $\N_0 \Let \{0, 1, 2, \ldots\}$ for the positive and non-negative integers, respectively. 
$\R$ denotes the set of real numbers. 
We use calligraphic letters, as in $\op{P}$ or $\op{C}$, for matrices/operators on the space of Q-functions, with $\op{I}$ denoting the identity matrix/operator.  
For a vector $V\in\R^d$, we use $\|V\|_p$ to denote its $p$-norm for $p\in\{1,2,\infty\}$. 
For a matrix $\mathcal{M} \in \R^{d\times m}$, we use $\|\mathcal{M}\|_{\infty}$ for its operator $\infty$-norm, that is, the maximum absolute row sum of $\mathcal{M}$.
For any finite set $\set{S}$, $\Delta(\set{S})$ denotes the probability simplex over $\set{S}$. 
For any set $\set{S}$, $\e\in\R^{\set{S}}$ is the all-ones vector/function with $\e(s) = 1$ for all $s\in\set{S}$. 
Finally, $\E$ denotes the expectation operator.\enlargethispage{\baselineskip}

\section{Preliminaries}\label{sec:preliminaries}

\subsection{MDPs, policy evaluation, and control}

In this paper, we consider an infinite-horizon discounted MDP $\langle \set{S}, \set{A}, \prob, r, \gamma \rangle$ consisting of
(i)~a finite state space~$\set{S}$,
(ii)~a finite action space~$\set{A}$,
(iii)~a transition probability kernel~$\prob$ such that $\prob(\cdot \vert s, a) \in \Delta(\set{S})$ gives the distribution of the next state for each state-action pair~$(s, a) \in \set{S} \times \set{A}$,
(iv)~a reward function~$r\colon \set{S} \times \set{A} \to \R$, which is bounded and represents the immediate reward received for each state-action pair, and
(v)~a discount factor~$\gamma \in (0,1)$.

Let $\set{Q} \Let \{Q \mid Q\colon \set{S} \times \set{A} \to \R\}$ denote the set of (Q-)functions that map $(s, a) \in \set{S}\times\set{A}$ to real values, and
$\set{\Pi} \Let \{\pi \mid \pi\colon \set{S} \to \Delta(\set{A})\}$ denote the set of stochastic policies that map states to distributions over actions.
We now define several key functions and operators.

Given a Q-function~$Q\in \set{Q}$, we define the set~$\set{\Pi}_Q$ of greedy policies with respect to $Q$ by
\[
	\set{\Pi}_Q \Let \{\pi \in \set{\Pi} \mid \pi(a \vert s) > 0 \Rightarrow a \in \argmax_{b \in \set{A}} Q(s, b),\, \forall (s, a) \in \set{S} \times \set{A}\}.
\]
We use $\pi_Q$ to denote a generic member of $\set{\Pi}_Q$.
We note that, in our analysis, the particular choice of greedy policy is inconsequential.

Given a policy $\pi \in \set{\Pi}$, the \emph{transition operator}~$\op{P}^{\pi}\colon \set{Q} \to \set{Q}$ is defined by
\begin{align*}
	\op{P}^{\pi} Q (s, a)
	\Let \ssum_{(s', a')\in \set{S} \times \set{A}} \pi(a' \vert s') \prob(s' \vert s, a) Q(s', a'),
	\quad \forall (s, a) \in \set{S} \times \set{A}.
\end{align*}
That is, $\op{P}^{\pi}$ propagates $Q$-functions through the Markov chain induced by $\pi$ (in the reverse direction with respect to $\prob$). 

The two main problems concerning an MDP are then as follows:
\begin{itemize}
	\item[(i)] \emph{Policy evaluation:} Given a policy $\pi \in \set{\Pi}$, the goal is to find the state-action values of $\pi$ defined by $Q^\pi \Let \sum_{k \geq 0} (\gamma \op{P}^\pi)^k r = (\op{I} - \gamma \op{P}^{\pi})^{-1} r$.
	In particular, the value $Q^\pi$ of $\pi$ admits the characterization $\op{T}^\pi Q^\pi = Q^\pi$, i.e., the fixed point of the \emph{Bellman operator} $\op{T}^\pi\colon \set{Q} \to \set{Q}$ defined by $\op{T}^\pi Q \Let \gamma \op{P}^\pi Q + r$.
	\item[(ii)] \emph{Control:} The goal is to find an optimal policy~$\pi^\star \in \set{\Pi}$ with the corresponding optimal value function
	      \[
		      Q^\star(s, a) \Let \sup_{\pi \in \set{\Pi}} Q^\pi(s, a), \qquad \forall (s, a) \in \set{S} \times \set{A},
	      \]
	      such that $Q^{\pi^\star} = Q^\star$.
	      The optimal value $Q^\star$ admits the characterization $\op{T} Q^\star = Q^\star$, i.e., the fixed point of the \emph{Bellman optimality operator} $\op{T}\colon \set{Q} \to \set{Q}$ defined by $\op{T} Q \Let \op{T}^{\pi_Q} Q$.
\end{itemize}

Both Bellman operators are $\gamma$-contractions under the $\infty$-norm, and thus repeated application of these operators to any $Q$-function converges to their respective fixed points $ (\op{T}^\pi)^k Q \to Q^\pi$ and $ (\op{T})^k Q \to Q^\star$ as $k \to \infty$ for all $Q\in\set{Q}$.
This is indeed the basis for the corresponding value iteration algorithms.

Although our primary focus is on off-policy value learning, we first introduce the classical algorithms for the control problem, since they provide the operator view from which we construct the preconditioned operators underlying our value learning algorithms.
The two classical algorithms for solving the control problem of an MDP are
\begin{equation*}
	\begin{array}{ll}
		\textbf{value iteration (VI):}  & Q_{+} = \op{T} Q = Q+\op{I} (\op{T} Q - Q),\\[1ex]
		\textbf{policy iteration (PI):} & Q_{+} = \lim_{k\to\infty} (\op{T}^{\pi_Q})^{k} Q = Q + (\op{I} - \gamma \op{P}^{\pi_Q})^{-1} (\op{T} Q - Q).
	\end{array}
\end{equation*}
Both VI and PI admit geometric convergence to $Q^\star$ in the $\infty$-norm at the same rate $\gamma$.
This rate, however, does not capture the practical speed of PI.
In finite MDPs, exact PI terminates after finitely many iterations~\cite{Scherrer2016improved} and typically reaches a neighborhood of $Q^\star$ in far fewer iterations than VI.
The speed-up in PI is attributed to its ``preconditioner''~$(\op{I} - \gamma \op{P}^{\pi_Q})^{-1}$~\cite{kolarijani2026quasi}, which plays the role of the inverse Jacobian when PI is viewed as Newton's method for the Bellman equation~\cite{ok2025control}.
This preconditioner fully propagates the Bellman residual $(\op{T}Q - Q)$ along the Markov chain induced by $\pi_Q$, thereby performing a  \emph{complete} policy evaluation step.
By contrast, VI employs the identity preconditioner~$\op{I}$ and hence performs only a \emph{partial} policy evaluation step.
This dichotomy motivates the search for a preconditioner that captures sufficient structure of the transition dynamics to accelerate convergence without requiring full policy evaluation.

Several algorithms interpolate between VI and PI. These algorithms mainly fall under the umbrella of \emph{Generalized Policy Iteration} (GPI) algorithms~\cite[\S4.6]{sutton2018reinforcement}.
To elaborate, for a policy $\pi$, let us define the \emph{$n$-step, $\lambda$-discounted preconditioner operator}~$\op{C}_{n,\lambda}^{\pi}\colon \set{Q} \to \set{Q}$ by
\begin{equation*}
	\label{eq:precond-def}
	\op{C}_{n,\lambda}^{\pi} \Let \ssum_{k = 0}^{n-1} (\lambda \gamma \op{P}^{\pi})^k, \quad \pi\in\set{\Pi},\ n\in \N, \ \lambda\in(0,1].
\end{equation*}
Let us also define the corresponding \emph{on-policy operator} $\op{T}^{\pi}_{n,\lambda}\colon \set{Q} \to \set{Q}$ by
\begin{align}
	\label{eq:on-multi-lambda}
	\op{T}_{n,\lambda}^{\pi} Q & \Let Q + \op{C}_{n,\lambda}^{\pi} (\op{T}^\pi Q - Q), \quad \pi\in\set{\Pi},\ n\in \N,\ \lambda \in (0,1].
\end{align}
Observe that $\op{C}_{1, \lambda}^{\pi} = \op{I}$ (and hence $\op{T}_{1,\lambda}^{\pi} = \op{T}^{\pi}$) and $\op{C}_{\infty, \lambda}^{\pi} = (\op{I}-\lambda\gamma\op{P}^{\pi})^{-1}$ for all $\lambda \in (0,1]$.
Then, two canonical examples of GPI are
\begin{equation*}\label{eq:MPI-lambdaPI-def}
	\begin{array}{ll}
		\textbf{modified policy iteration (MPI):}       & Q_{+} = \op{T}_{n,1}^{\pi_Q} Q,            \\[1ex]
		\textbf{$\bm{\lambda}$-policy iteration (LPI):} & Q_{+} = \op{T}_{\infty,\lambda}^{\pi_Q} Q.
	\end{array}
\end{equation*}
The MPI operator essentially performs a truncated \emph{policy evaluation} step determined by $n$,
whereas LPI discounts the contribution of the Bellman residual by a factor of $\lambda$.

In model-free RL, we only have access to sampling oracles (sample-based estimators) for single-step operators (i.e., with identity preconditioner) such as $\op{T}^\pi$ (e.g., SARSA~\cite[\S6.4]{sutton2018reinforcement}) and $\op{T}$ when $\argmax_{a \in \set{A}} Q(\cdot, a)$ is available (e.g., Q-learning~\cite{watkins1992q}).
In $n$-step on-policy algorithms, one uses the same policy ($\pi$ in the case of policy evaluation or $\pi_Q$ in the case of control) for both the preconditioner~$\op{C}_{n,\lambda}^{\pi} $ and the Bellman residual~$(\op{T}^\pi Q - Q)$, as can be seen in~\eqref{eq:on-multi-lambda}.
However, in the off-policy setting, the transition operator $\op{P}^\pi$ of the target policy~$\pi$ or its sampling oracle is unavailable or undesirable (e.g., for better exploration).
Instead, we use the transition operator $\op{P}^{\mu}$ of the behavior policy~$\mu$ to propagate the Bellman residual $(\op{T}^\pi Q - Q)$ of the target policy~$\pi$.
In order to account for the discrepancy between $\mu$ and $\pi$, one may also modify the transition operator~$\op{P}^{\mu}$ by including a correction term.
To be precise, following \citet{munos2016safe}, for a correction function~$c\colon \set{S} \times \set{A} \to \R$, we define the \textit{corrected} transition operator~$\op{P}^{c, \mu}\colon \set{Q} \to \set{Q}$ by
\begin{align*}
	\op{P}^{c, \mu} Q (s, a)
	\Let \ssum_{(s', a')\in \set{S} \times \set{A}} c(s', a') \mu(a' \vert s') \prob(s' \vert s, a) Q(s', a'),
	\quad \forall (s, a) \in \set{S} \times \set{A}.
\end{align*}
Depending on the choice of $c$, off-policy operators fall into two categories: conservative and non-conservative.
In the conservative category, a canonical choice for $c$ is the \emph{importance sampling} (IS) ratio $c(s, a) = \nicefrac{\pi(a \vert s)}{\mu(a \vert s)}$ when both distributions are known.
IS-based estimators, although unbiased, suffer from high variance~\cite{munos2016safe}.
To address this issue, many low-variance alternatives for $c$ have been proposed without hindering value learning~\cite{munos2016safe, precup2000eligibility}.
Non-conservative operators, on the other hand, use a fixed $c$, such as $c=\e \in \set{Q}$, i.e., the all-ones function.
As a result, these operators do not require knowledge of $\mu$ and are less restrictive.
The Harutyunyan~\cite{harutyunyan2016q} and Peng \& Williams~\cite{peng1994incremental, kozuno2021revisiting} operators are examples of non-conservative off-policy operators.

\subsection{Truncated Harutyunyan operator}

In this work, we focus on \emph{non-conservative} $n$-step off-policy operators that utilize the preconditioners defined over a behavior policy.
To this end, for a behavior policy~$\mu$ and a target policy~$\pi$,  we define the \emph{$\bm{n}$-step Harutyunyan operator} $\op{R}^{\mu,\pi}_{n,\lambda}\colon \set{Q} \to \set{Q}$ by
\begin{align}
	\label{eq:off-multi-lambda}
	\op{R}_{n,\lambda}^{\mu,\pi} Q & \Let Q + \op{C}_{n,\lambda}^{\mu} (\op{T}^\pi Q - Q), \quad \mu,\pi\in\set{\Pi},\ n\in\N,\ \lambda \in (0,1].
\end{align}
We next define the corresponding Harutyunyan iterations that form the value update algorithms by
\begin{equation}\label{eq:HQL-truncated-def}
	\begin{array}{rlr}
		\HIpi : & Q_{\ell+1} = \op{R}_{n,\lambda}^{\mu,\pi} Q_\ell, \quad \ell\in \N_0,\; Q_0 \in \set{Q};            & \text{(policy evaluation)} \\[2ex]
		\HI:    & Q_{\ell+1} = \op{R}_{n,\lambda}^{\mu,\pi_{Q_{\ell}}} Q_\ell, \quad \ell\in \N_0,\; Q_0 \in \set{Q}. & \text{(control)}
	\end{array}
\end{equation}
We note that the \emph{non-truncated} operator~$\op{R}_{\infty,\lambda}^{\mu,\pi}$ corresponds to \textit{Harutyunyan's original Q($\lambda$) (HQL)}~\cite{harutyunyan2016q}, which introduces the so-called \emph{off-policy correction} terms.
In this work, we focus on the generic $n$-step truncated operator in~\eqref{eq:off-multi-lambda} as it is widely used in practice~\cite{mnih2016asynchronous, hessel2018rainbow, espeholt2018impala, kapturowski2019r2d2}.
We start with the convergence guarantees for the truncated operator.
For this purpose, following \citet{harutyunyan2016q}, we define the maximum statewise difference $\epsilon \in [0, 2]$ between two policies as
\begin{equation*}
	\label{eq:epsilon}
	\epsilon \Let \max_{s \in \set{S}} \norm{\mu(\cdot \vert s) - \pi(\cdot \vert s)}_1, \quad \mu,\pi\in\set{\Pi}.
\end{equation*}
Moreover, given scalars $\gamma$ and $\lambda$, define the $\epsilon$-dependent coefficient
\begin{equation*}
    \label{eq:eta}
	\eta_{\epsilon} \Let \frac{\gamma \bigl(1+\lambda(\epsilon -1)\bigr)}{1 - \lambda\gamma},
\end{equation*}
and the corresponding boundary values
\begin{equation*}
    \label{eq:eta_boundary}
    \eta_0 \Let  \frac{\gamma(1-\lambda)}{1-\lambda\gamma} \quad \text{and} \quad \eta_2 \Let \frac{\gamma (1+\lambda)}{1 - \lambda\gamma}.
\end{equation*}
The preceding coefficient characterizes the geometric convergence rates of the original Harutyunyan operators (with $n=\infty$) to their corresponding fixed points.
In policy evaluation, the iteration converges to $Q^{\pi}$ with rate $\eta_{\epsilon}$ if $\epsilon < \frac{1-\gamma}{\gamma\lambda}$.
In particular, $\eta_0$ is the fastest rate corresponding to the \emph{on-policy} setting with $\epsilon = 0$.
In control, the iteration converges to $Q^{\star}$ with rate $\eta_2$ if $\lambda < \frac{1-\gamma}{2\gamma}$~\cite{harutyunyan2016q}.
The following lemma characterizes the convergence rate of the truncated, $n$-step Harutyunyan operators.
\begin{Lem}[Convergence rate of HI($\pi$) \& HI]
	\label{lem:convergence-h}
	Consider the $n$-step Harutyunyan operator~\eqref{eq:off-multi-lambda} and the corresponding iterations in~\eqref{eq:HQL-truncated-def}.
	\begin{itemize}
		\item \emph{(Policy evaluation)}
            The iterates of \emph{HI($\pi$)} satisfy $\norm{Q_\ell - Q^\pi}_\infty \leq \alpha_{\mathrm{e}}^\ell \norm{Q_0 - Q^\pi}_\infty$ for all $\ell \in \N_0$, where
            \[
            \alpha_{\mathrm{e}} \Let (\gamma\lambda)^{n-1}\bigl(\gamma - \eta_{\epsilon}\bigr) + \eta_{\epsilon}.
            \]
            In particular, for $n\geq2$, $\alpha_{\mathrm{e}} < 1$ and hence $Q_\ell \to Q^\pi$ if
            \begin{equation}\label{eq:epsilon-HI-pi}
            \epsilon < \epsilon_{\max} \Let \frac{1-\gamma}{\gamma\lambda}\cdot\frac{1-(\gamma\lambda)^n}{1-(\gamma\lambda)^{n-1}}. 
            \end{equation}
		\item \emph{(Control)}
            The iterates of \emph{HI} satisfy $\norm{Q_\ell - Q^\star}_\infty \leq \beta_{\mathrm{c}}^\ell \norm{Q_0 - Q^\star}_\infty$ for all $\ell \in \N_0$, where
            \[
            \beta_{\mathrm{c}} \Let \eta_2 - \frac{(1+\gamma)(\gamma\lambda)^n}{1-\gamma\lambda}.
            \]
            In particular, $\beta_{\mathrm{c}} < 1$ and hence $Q_\ell \to Q^\star$ if 
	            \begin{equation}\label{eq:lambda-HI-imp}
                \lambda < \frac{1-\gamma}{2\gamma} + \frac{(1+\gamma)(\gamma\lambda)^n}{2\gamma}.
	            \end{equation}
			\end{itemize}
\end{Lem}

\begin{Rem}[Convergence of HI] 
The implicit condition~\eqref{eq:lambda-HI-imp} can be characterized explicitly as 
\begin{equation}\label{eq:lambda-HI-exp}
    \lambda < \lambda_{\max} \Let \left\{\begin{array}{ll}
         1/\gamma, &  n=1, \\
         \overline{x}/\gamma \in\big(\frac{1-\gamma}{2\gamma},\frac{1-\gamma}{\gamma(1+\gamma)}\big], &  n\geq 2,
    \end{array}\right.
\end{equation}
where, for $n\geq 2$, $\overline{x}>0$ uniquely solves the equation $\sum_{i=1}^{n-1} x^{i} = \frac{1-\gamma}{1+\gamma}$.
Observe that the upper bound $\lambda_{\max}$ is a decreasing function of $n$. 
Moreover, $\lambda_{\max} > 1$, i.e., the condition becomes trivial, if any of the following conditions holds:
\begin{equation*}
    \text{\emph{(i)} }n=1, \quad  \text{\emph{(ii)} }\gamma \leq\frac{1}{3},\quad  \text{\emph{(iii)} }n < \frac{\log(3\gamma-1)- \log(1+\gamma)}{\log\gamma}\ \;\; \text{for} \;\; \gamma>\frac{1}{3}.
\end{equation*}
\end{Rem}

Some remarks are in order regarding the preceding result.
First, observe that the extreme values of $n$ recover the standard convergence results as expected.
For $n=1$, the rates $\alpha_{\mathrm{e}}$ for policy evaluation and $\beta_{\mathrm{c}}$ for control are both equal to $\gamma$.
Moreover, when $n=1$, the convergence bounds become trivial (i.e., $\epsilon_{\max} = \infty$ and $\lambda_{\max} = \gamma^{-1} > 1$), recovering the convergence behavior of the standard value iteration (indeed, $\op{R}_{1,\lambda}^{\mu,\pi} = \op{T}^{\pi}$ for all $\mu\in\set{\Pi}$ and $\lambda\in(0,1]$).
On the other hand, as $n \to \infty$, we have $\alpha_{\mathrm{e}} \to \eta_{\epsilon}$ and $\epsilon_{\max} \to \frac{1-\gamma}{\gamma\lambda}$ for policy evaluation and $\beta_{\mathrm{c}} \to \eta_2$ and $\lambda_{\max} \to \frac{1-\gamma}{2\gamma}$ for control, recovering the convergence properties of the original non-truncated operators.

An interesting observation is that for both policy evaluation and control, truncation improves the stability of the iteration by increasing the upper bounds $\epsilon_{\max}$ and $\lambda_{\max}$ in the conditions~\eqref{eq:epsilon-HI-pi} and~\eqref{eq:lambda-HI-exp} for convergence.
Indeed, similar to the role of $\lambda$, the truncation length $n$ controls the reliance on Bellman residuals from future steps.
Hence, shortening $n$ leads to a more stable iteration that can handle a larger policy discrepancy.

Regarding the convergence rate, in control, we have that $\gamma \leq \beta_{\mathrm{c}} < \eta_2$ for all $n\in\N$, indicating that truncation always improves the convergence rate.
However, in policy evaluation, we have two regimes based on the policy discrepancy.
If $\epsilon < 1 - \gamma$, then $\eta_{\epsilon} < \alpha_{\mathrm{e}} \leq \gamma$ for all $n\in\N$, meaning that truncation worsens the convergence rate.
In contrast, if $\epsilon > 1 - \gamma$, then $\gamma \leq \alpha_{\mathrm{e}} < \eta_{\epsilon}$ for all $n\in\N$, implying that truncation improves the convergence rate.
This can be explained by the fact that as $\epsilon$ vanishes, the iteration becomes on-policy with convergence rate $\eta_0 < \gamma$ for all $\lambda \in (0,1]$.
At the operator level, Lemma~\ref{lem:convergence-h} thus provides explicit convergence rates for the truncated $n$-step Harutyunyan operator in both policy evaluation and control.

\subsection{Limiting kernel, aperiodicity, and unichains}

We finish this section by defining the limiting kernel and reviewing aperiodicity and unichains.
So, let us fix a policy $\nu\in\set{\Pi}$ and consider the transition operator $\op{P}^\nu$ of the Markov chain induced by $\nu$. 
We define the \emph{limiting kernel (LK)}~$\op{P}_\infty^\nu\colon \set{Q} \to \set{Q}$ to be the corresponding Ces\`{a}ro limit~\cite[App.~A.4]{puterman1994markov}, that is,
\begin{equation*}
	\label{eq:limiting-kernel-generic}
	\op{P}_\infty^\nu \Let \lim_{k\to\infty} \frac{1}{k} \ssum_{j<k} (\op{P}^\nu)^j.
\end{equation*}
We note that the limiting kernel exists for every finite chain. 
Moreover, it is idempotent and absorbs $\op{P}^\nu$ from either side~\cite[\S3.2]{boldi2006graph}, that is, 
\[
\op{P}_\infty^\nu  \op{P}_\infty^\nu = \op{P}_\infty^\nu   \op{P}^\nu = \op{P}^\nu   \op{P}_\infty^\nu = \op{P}_\infty^\nu.
\]
Let us also define the finite-horizon limiting-kernel residual by
\[
\delta^\nu_n \Let \norm{(\op{P}^{\nu})^n - \op{P}_\infty^\nu}_{\infty} \in [0,2], \quad n\in\N_0.
\]
Note that $\delta^\nu_n \leq 2$, independent of $\nu$ and $n$. 
This upper bound, however, can be significantly improved for a certain class of chains, as we discuss next.

The transition operator $\op{P}^\nu$ is called \emph{aperiodic} if every recurrent class of the induced chain is aperiodic.
In this case, the powers of $\op{P}^\nu$ converge as $k\to\infty$~\cite[\S4.3.5]{gallager2013stochastic}, and the limiting kernel is equivalently given by
\begin{equation*}
	\label{eq:limiting-kernel}
	\op{P}_\infty^\nu = \lim_{k\to\infty} (\op{P}^\nu)^k.
\end{equation*}
For an aperiodic chain, the limiting kernel consists of stationary distributions of the chain on the recurrent classes.
That is, $\op{P}_\infty^\nu Q (s, a)$ is a weighted combination of the stationary averages of $Q$ across the reachable recurrent classes from $(s, a)$.
It follows from \citet[\S4.4.2]{gallager2013stochastic} that for an aperiodic $\op{P}^\nu$, we can define the \emph{mixing rate} $\sigma_{\nu} \in (0,1)$ and \emph{mixing constant} $c_{\nu}<\infty$ such that 
\begin{equation}
	\label{eq:p-infty-convergence}
	\delta^\nu_n \leq c_{\nu} \sigma_{\nu}^n, \quad \forall n\in\N_0.
\end{equation}
We note that $\sigma_{\nu}\in(\sigma_{\nu}^{\star},1)$, where $\sigma_{\nu}^{\star}$ is the largest modulus among the eigenvalues of $\op{P}^\nu$ other than $1$, with $\sigma_{\nu}^{\star} \Let 0$ if no such eigenvalue exists. 
Similarly, we can define the \emph{mixing time} $n_{\nu} \in \N$ of $\op{P}^\nu$ by
\begin{equation}
    \label{eq:tail-improvement-threshold}
    n_\nu \Let \min\{n \in \N \mid \delta^\nu_n < 1\}.
\end{equation}

We call the chain induced by $\nu$ a \emph{unichain} if it has a single recurrent class, possibly together with transient states~\cite[Def.~4.3.2]{gallager2013stochastic}. 
When the induced chain is a unichain, the limiting kernel $\op{P}_\infty^\nu$ reduces to a rank-one operator, given by $\op{P}_\infty^\nu Q \Let \e\,\E_{(s,a)\sim d_\nu}[Q(s,a)]$, where $d_\nu \in \Delta(\set{S}\times\set{A})$ is the unique stationary distribution.
If the unichain is also aperiodic, then $(\op{P}^\nu)^k Q \to \e\,\E_{(s,a)\sim d_\nu}[Q(s,a)]$ as $k\to\infty$ for every $Q\in\set{Q}$.

\section{Tail-completed operators}\label{sec:acc-ops}
In the previous section, we defined the truncated off-policy policy evaluation operator~$\op{R}_{n,\lambda}^{\mu,\pi}$ in~\eqref{eq:off-multi-lambda} for a given target policy~$\pi$ and behavior policy~$\mu$.
In this section, we present \emph{tail-completed} versions of these operators using \emph{limiting-kernel} preconditioners, together with the corresponding convergence guarantees and acceleration over the truncated operators under some conditions.

\subsection{Limiting-kernel preconditioner}\label{sec:spec-pre}

Many on- and off-policy operators build their preconditioner by approximating the PI preconditioner via $n$-step truncation and/or averaging, as we discussed in the previous section.
In contrast, operator splitting (OS) and the corresponding OS-VI algorithm of \citet{rakhsha2022operator} leverage the decomposition of complete \emph{policy evaluation}, i.e., the computation of
$\op{C}_{\infty,1}^{\pi}$, through inversion of ``simpler'' surrogates.
To be precise, given a policy~$\pi$ and an approximate transition operator~$\wh{\op{P}}^\pi$, OS-VI introduces the \emph{Varga operator}~$\op{V}^{\pi}\colon \set{Q} \to \set{Q}$ defined by
\begin{equation}
	\label{eq:varga}
	\op{V}^{\pi} Q \Let
	(\op{I} - \gamma \wh{\op{P}}^{\pi})^{-1}
	\bigl(r + \gamma(\op{P}^{\pi} - \wh{\op{P}}^{\pi}) Q\bigr).
\end{equation}
That is, $\op{P}^\pi$ is replaced by the approximation $\wh{\op{P}}^\pi$ in the preconditioner of $\op{V}^{\pi}$.
Notably, under $\pi=\pi_Q$, the Varga operator reduces to VI, LPI, and PI for
$\wh{\op{P}}^{\pi_Q} = 0$,
$\wh{\op{P}}^{\pi_Q} = \lambda \op{P}^{\pi_Q}$, and
$\wh{\op{P}}^{\pi_Q} = \op{P}^{\pi_Q}$, respectively.

The choice of $\wh{\op{P}}^\pi$ is crucial for the convergence rate of the Varga operator.
However, since the preconditioner involves inverting $(\op{I} - \gamma \wh{\op{P}}^\pi)$ and must be applicable in a model-free setting, the choice of $\wh{\op{P}}^\pi$ must ensure both tractable invertibility and estimability.
The former condition motivated the low-rank approximations of $\op{P}^\pi$ in DDVI~\cite{lee2024deflated}, which approximate $\op{P}^\pi$ via its leading spectral components to invert the preconditioner efficiently.
Ensuring estimability further constrains the choice of approximation.
For instance, R1-QL~\cite{kolarijani2025rank} forms a rank-one surrogate from a power-method estimate of a stationary distribution of $\op{P}^\pi$.

We instead approximate $\op{P}^\pi$ by its limiting kernel $\op{P}_\infty^\pi$. 
We show that this approximation satisfies both conditions put forward previously, namely, tractable invertibility and estimability from the samples.
We defer the discussion on estimation to a later section and focus on the resulting operator and its convergence properties below. 

Observe that, since the limiting kernel $\op{P}_\infty^\pi$ is idempotent, the inversion in the preconditioner with $\op{P}_\infty^\pi$ admits the closed form 
\[
	(\op{I} - \gamma \op{P}_\infty^\pi)^{-1} = \op{I} + \frac{\gamma}{1 - \gamma} \op{P}_\infty^\pi.
\]
Then, substituting $\wh{\op{P}}^\pi$ with $\op{P}_\infty^\pi$ in the Varga operator~\eqref{eq:varga} yields the \emph{limiting-kernel Varga operator} $\wt{\op{V}}^{\pi}\colon \set{Q} \to \set{Q}$ given by
\begin{align*}
	\wt{\op{V}}^{\pi} Q & \Let
	Q + \Bigl(\op{I} + \frac{\gamma}{1 - \gamma} \op{P}_\infty^\pi\Bigr)(\op{T}^{\pi} Q - Q).
\end{align*}
As we discussed above, when the chain induced by $\pi$ is a unichain, the limiting kernel reduces to a rank-one operator. 
With $\pi=\pi_Q$, the corresponding rank-one LK Varga operator is thus the same as the R1-VI operator of \citet{kolarijani2025rank}.
We note that the R1-VI algorithm of \citet{kolarijani2025rank} approximates the stationary distribution of the current policy-induced transition matrix by a single step of the power method in each iteration and is guaranteed to converge with at least the same geometric rate $\gamma$ as VI.

The limiting-kernel approximation enjoys two advantages.
First, let $\op{F}\colon\set{Q}\to\set{Q}$ be any operator such that $\op{F}Q$ admits a sample-based estimator, e.g., the Bellman residual $\op{F}Q = \op{T}^\pi Q - Q$.
Then, $\op{P}_\infty^\pi\op{F}Q$ can also be estimated from rollouts.
We elaborate on this in Section~\ref{sec:val-learn}.
Second, it is compatible with $n$-step truncations.
Given a policy $\pi$, the infinite sum in the preconditioner~$\op{C}^{\pi}_{\infty, \lambda}$ can be split into two parts as follows
\begin{align*}
	\op{C}^{\pi}_{\infty, \lambda} & = \ssum_{k=0}^{n-2} (\lambda \gamma \op{P}^\pi)^k + (\lambda \gamma \op{P}^\pi)^{n-1} (\op{I} - \lambda \gamma \op{P}^\pi)^{-1}.
\end{align*}
Approximating $(\op{I} - \lambda \gamma \op{P}^\pi)^{-1}$ by $(\op{I} - \lambda \gamma \op{P}_\infty^\pi)^{-1}$, we derive the \emph{limiting-kernel (LK) preconditioner}~$\wt{\op{C}}_{n,\lambda}^{\pi}\colon \set{Q} \to \set{Q}$ given by
\begin{align}\label{eq:LK-precondition}
	\wt{\op{C}}^{\pi}_{n, \lambda} & \Let \op{C}^{\pi}_{n, \lambda} + \frac{(\lambda\gamma)^n}{1 - \lambda\gamma} \op{P}_\infty^\pi, \quad n\in\N,\, \lambda\in(0,1].
\end{align}
Observe that the effect of the limiting-kernel modification diminishes as $n \to \infty$. 
Indeed, $\wt{\op{C}}^{\pi}_{\infty, \lambda} = \op{C}^{\pi}_{\infty, \lambda}$ for all $\lambda \in (0,1]$.
This modification can be interpreted as \emph{tail completion}, since the powers of $\op{P}^{\pi}$ in the tail omitted by $\op{C}^{\pi}_{n, \lambda}$ are approximated by $\op{P}_\infty^\pi$ in $\wt{\op{C}}^{\pi}_{n, \lambda}$, that is, 
\begin{align}\label{eq:LK-precondition-alternative}
	\wt{\op{C}}^{\pi}_{n, \lambda} = \ssum_{k=0}^{n-1} (\lambda \gamma \op{P}^\pi)^k + \ssum_{l \geq  n} (\lambda \gamma \op{P}_\infty^\pi)^l.
\end{align}

If the policy $\pi$ is such that $\op{P}^\pi$ is aperiodic, then $(\op{P}^\pi)^{n}$ approaches $\op{P}_\infty^\pi$ as $n$ increases,  
which makes the limiting-kernel approximation in~\eqref{eq:LK-precondition} increasingly accurate for larger $n$.
As a result, if $n$ is large enough, the LK preconditioner $\wt{\op{C}}^{\pi}_{n, \lambda}$ is strictly closer to $\op{C}^{\pi}_{\infty, \lambda}$ than the truncated preconditioner~$\op{C}^{\pi}_{n, \lambda}$ in terms of the operator $\infty$-norm.
The following lemma formalizes this statement.

\begin{Lem}[Preconditioner tail completion]
	\label{lem:lk-aperiodicity}
	Let $\pi\in\set{\Pi}$ and suppose that $\op{P}^\pi$ is aperiodic with mixing rate $\sigma_{\pi} \in (0,1)$ and mixing time $n_{\pi} \in\N$. 
	Define $\delta_n \Let \|\op{C}^{\pi}_{n, \lambda} - \op{C}^{\pi}_{\infty, \lambda}\|_{\infty}$ and $\wt{\delta}_n \Let \|\wt{\op{C}}^{\pi}_{n, \lambda} - \op{C}^{\pi}_{\infty, \lambda}\|_{\infty}$.
	Then, $\wt{\delta}_n / \delta_n = \mathcal{O}(\sigma_{\pi}^n)$.
	In particular, $\wt{\delta}_n < \delta_n$ for all $n \geq n_\pi$.
\end{Lem}

If the chain induced by $\pi$ is a unichain and $\lambda=1$, the LK preconditioner~\eqref{eq:LK-precondition} reduces to the preconditioner (gain matrix) introduced by \citet[Rem.~3.4]{kolarijani2025rank} for rank-one modified policy iteration (R1-MPI).
In contrast, the LK formulation accommodates behavior chains with multiple recurrent classes and admits a scalable sample-based estimator, which we develop in Section~\ref{sec:val-learn}.

\subsection{Off-policy operators}\label{sec:off-policy-ops}

Using the LK preconditioner introduced in~\eqref{eq:LK-precondition}, we now define the
corresponding \emph{limiting-kernel (LK) Harutyunyan operator} $\wt{\op{R}}^{\mu,\pi}_{n,\lambda}\colon \set{Q} \to \set{Q}$ for a behavior policy~$\mu$ and a target policy~$\pi$ by
\begin{align*}
	\label{eq:lk-on-n-lambda}
	\wt{\op{R}}_{n,\lambda}^{\mu,\pi} Q
	 & \Let Q + \Bigl(\op{C}^{\mu}_{n, \lambda} + \frac{(\lambda\gamma)^n}{1 - \lambda\gamma} \op{P}_\infty^\mu\Bigr)
	(\op{T}^\pi Q - Q),
	\quad \mu,\pi\in\set{\Pi},\ n\in\N,\ \lambda \in (0,1].
\end{align*}
This is the extension of the $n$-step Harutyunyan operator~$\op{R}_{n,\lambda}^{\mu,\pi}$ with LK preconditioning.
For the rest of this section, we omit the on-policy operator and its policy evaluation analysis, as the off-policy operator reduces to the on-policy operator when $\mu = \pi$ (i.e., $\epsilon = 0$).

Intuitively, the $n$-step Harutyunyan operator in HI, through $\op{C}^\mu_{n,\lambda}$, propagates the next $n-1$ Bellman errors through the Markov chain induced by the behavior policy $\mu$.
The LK preconditioner, in turn, performs tail completion by approximating the omitted tail with the limiting kernel from step $n$ onward, propagating the Bellman error according to the stationary structure of the recurrent classes reachable from each state-action pair.
Figure~\ref{fig:backup} illustrates the corresponding backup diagram of the LK Harutyunyan operator, depicting the resulting target value estimator.

\input{figure/lk_operator_diagram.tex}

We now introduce the corresponding iterations and analyze their convergence properties. 
To that end, let us define the LK iterations
\begin{equation}\label{eq:LKH-iterations}
	\begin{array}{rlr}
		\LKHIpi : & Q_{\ell+1} = \wt{\op{R}}_{n,\lambda}^{\mu,\pi} Q_\ell, \quad \ell\in \N_0,\; Q_0 \in \set{Q};            & \text{(policy evaluation)} \\[2ex]
		\LKHI:    & Q_{\ell+1} = \wt{\op{R}}_{n,\lambda}^{\mu,\pi_{Q_{\ell}}} Q_\ell, \quad \ell\in \N_0,\; Q_0 \in \set{Q}. & \text{(control)}
	\end{array}
\end{equation}

Our first result concerns the convergence of LKHI($\pi$) for policy evaluation under a generic behavior policy $\mu$.
Recall that $\delta^\mu_n$ denotes the $n$-step limiting-kernel residual for the policy~$\mu$, and $\alpha_{\mathrm{e}}$ is the convergence rate of HI($\pi$) established in Lemma~\ref{lem:convergence-h}.
\begin{Thm}[Convergence of LKHI($\pi$)]
	\label{thm:conv-lkhi-pi}
	Consider the iterates of \emph{LKHI($\pi$)} in~\eqref{eq:LKH-iterations}. 
    We have $\norm{Q_\ell - Q^\pi}_\infty \leq \wt\alpha_{\mathrm{e}}^\ell \norm{Q_0 - Q^\pi}_\infty$ for all $\ell \in \N_0$, where $\wt\alpha_{\mathrm{e}} \Let \eta_{\epsilon} + (\gamma\lambda)^n  \delta^\mu_n$.
    In particular, $\wt\alpha_{\mathrm{e}} < 1$ and hence $Q_\ell \to Q^\pi$ if
            \begin{equation*}
            \label{eq:epsilon-LKHI-pi}
            \epsilon < \wt\epsilon_{\max} \Let \frac{1-\gamma}{\gamma\lambda} - \frac{1-\gamma\lambda}{\gamma\lambda} (\gamma\lambda)^n\delta^\mu_n.
            \end{equation*}
    Moreover, we have $\wt\alpha_{\mathrm{e}} < \alpha_{\mathrm{e}}$ if and only if $\delta^\mu_n < \frac{1-\gamma-\epsilon}{1-\gamma\lambda}$.
\end{Thm}

The preceding theorem characterizes the convergence rate $\wt\alpha_{\mathrm{e}}$ and the upper bound $\wt\epsilon_{\max}$ on the policy discrepancy for LKHI($\pi$).
Unlike HI($\pi$), whose convergence rate $\alpha_{\mathrm{e}}$ can be smaller or larger than $\eta_{\epsilon}$, LKHI($\pi$) has a convergence rate that is the same as or slower than that of the non-truncated operator, as $\wt\alpha_{\mathrm{e}} \geq \eta_{\epsilon}$.
Similarly, it admits a range of policy discrepancies that is no wider than that of the non-truncated operator, as $\wt\epsilon_{\max} \leq \frac{1-\gamma}{\gamma\lambda}$.
However, compared to HI($\pi$), which uses the truncated operator $\op{C}^{\mu}_{n,\lambda}$, LKHI($\pi$) can achieve a strictly smaller convergence rate $\wt\alpha_{\mathrm{e}} < \alpha_{\mathrm{e}}$ in the regime $\epsilon \in [0,1-\gamma)$, provided that the residual $\delta^\mu_n$ is sufficiently small.
Recall that, for policy evaluation, the non-truncated preconditioner $\op{C}^{\mu}_{\infty, \lambda}$ yields a strictly smaller convergence rate than the truncated counterpart $\op{C}^{\mu}_{n, \lambda}$ in the same regime for $\epsilon$.
Correspondingly, the LK preconditioner $\wt{\op{C}}^{\mu}_{n, \lambda}$, as an approximation of $\op{C}^{\mu}_{\infty, \lambda}$, can only inherit this improvement over the same regime. 

Note that the condition on the residual $\delta^\mu_n$ ensuring $\wt\alpha_{\mathrm{e}} < \alpha_{\mathrm{e}}$ may not hold for any $n\in\N$ for a generic behavior policy $\mu$.
However, for an \emph{aperiodic} behavior transition operator $\op{P}^\mu$, this is always the case for large enough $n$, given $\epsilon < 1-\gamma$. 
Indeed, by Lemma~\ref{lem:lk-aperiodicity}, for a large enough $n$, $\wt{\op{C}}^{\mu}_{n, \lambda}$ is a strictly better approximation of $\op{C}^{\mu}_{\infty, \lambda}$ than $\op{C}^{\mu}_{n, \lambda}$.
The following corollary characterizes this.

\begin{Cor}[Convergence of LKHI($\pi$) with aperiodicity]
	\label{cor:conv-lkhi-pi}
	Consider the setup of Theorem~\ref{thm:conv-lkhi-pi}. 
    Assume that $\op{P}^\mu$ is aperiodic with mixing rate $\sigma_{\mu} \in (0,1)$ and mixing constant $c_{\mu} > 0$. 
    Also, assume that $\epsilon < 1-\gamma$. 
    Then, $\wt\alpha_{\mathrm{e}} < \alpha_{\mathrm{e}}$ if $n > \log\big(c_\mu(1-\gamma\lambda)(1-\gamma-\epsilon)^{-1}\big)/\log(\sigma_\mu^{-1})$.
\end{Cor}

The preceding result indicates that the greatest improvement is for $\epsilon = 0$, that is, for the on-policy setting. The following remark discusses this in more detail.

\begin{Rem}[Convergence of on-policy LKHI($\pi$) with aperiodicity]
    \label{rem:conv-lkhi-pi}
    Under the assumption of aperiodicity, the greatest improvement in the convergence rate of LKHI($\pi$) over \emph{HI($\pi$)} is achieved in the on-policy setting where $\mu = \pi$ with $\epsilon = 0$.
    In particular, for $\epsilon = 0$ and $\lambda = 1$, \emph{LKHI($\pi$)} has a strictly smaller convergence rate than \emph{HI($\pi$)} for every $n \geq n_\pi$, with $n_\pi$ being the mixing time of $\op{P}^{\pi}$.
	This is expected in the on-policy setting, where tail completion approximates full policy evaluation of the target policy.
\end{Rem}

Next, we analyze the convergence of LKHI in control.
\begin{Thm}[Convergence of LKHI]
	\label{thm:convergence-control}
	Consider the iterates of \emph{LKHI} in~\eqref{eq:LKH-iterations}.
    We have
    \[
    \norm{Q_\ell - Q^\star}_\infty \leq \wt\beta_{\mathrm{c}}^\ell \norm{Q_0 - Q^\star}_\infty
    \]
    for all $\ell \in \N_0$, where $\wt\beta_{\mathrm{c}} \Let \eta_2$.
	In particular, $\wt\beta_{\mathrm{c}} < 1$ and hence $Q_\ell \to Q^\star$ if $\lambda < \frac{1-\gamma}{2\gamma}$.
\end{Thm}
The preceding theorem shows that LKHI has the same convergence rate and converges under the same conditions as the non-truncated Harutyunyan iteration.
In contrast to policy evaluation, the preceding bound does not establish a rate improvement for LKHI over HI.
This stems from the discontinuous greedy map $Q \mapsto \pi_Q$, under which small changes in $Q$ can perturb $\pi_Q$ arbitrarily, leaving the policy discrepancy $\epsilon$ unconstrained across iterations.
The analysis is therefore forced to accommodate the worst case $\epsilon = 2$, at which the LK and non-truncated preconditioners share the same operator norm and yield the common rate $\eta_2$.

Beyond the convergence properties established above, the applicability of LKHI hinges on whether it admits a tractable sample-based estimator from rollouts of the behavior chain.
In the following subsection, we construct such an estimator under a fixed behavior policy $\mu$ and establish almost-sure convergence of the resulting learning algorithm to $Q^\star$.

\subsection{Value estimation with a double backup}\label{sec:val-learn}

In the previous subsection, we showed that LKHI converges to optimal Q-values with the same rate as the non-truncated Harutyunyan iteration (that is, HQL) in control.
Here, we develop the corresponding learning algorithm, which we call \emph{Limiting-Kernel Q($\bm\lambda$)} (LKQL).
Our sample-based realization of the LKQL update depends on estimating an approximation of the limiting kernel $\op{P}_\infty^\mu$ of the behavior policy $\mu\in\set{\Pi}$. 
To be precise, for $\tau\in[0,1)$, we consider the approximation $\op{A}_\tau^\mu$ of $\op{P}_\infty^\mu$ defined by
\begin{equation*}
	\label{eq:approx-p-infty}
	\op{A}_\tau^\mu \Let (1 - \tau)(\op{I} - \tau \op{P}^\mu)^{-1}.
\end{equation*}
Observe that $\lim_{\tau \to 1} \op{A}_\tau^\mu = \op{P}_\infty^\mu$. 
Equipped with $\op{A}_\tau^\mu$, the estimation of the LK term reduces to a standard TD learning problem along a sampled rollout following policy $\mu$, which we develop next.

To this end, we introduce an auxiliary function $U\in \set{Q}$ tracking the long-run average TD error and defined by the fixed-point equation
\begin{equation}
	\label{eq:u-bellman}
	U= (1 - \tau)(\op{T} Q - Q) + \tau \op{P}^\mu U.
\end{equation}
Notice that the preceding equation is analogous to a Bellman equation with ``reward'' $(1 - \tau)(\op{T} Q - Q)$ and ``discount'' $\tau$ under the transition operator $\op{P}^\mu$.
As $\tau \to 1$, $U$ provides a more accurate approximation of $\op{P}_\infty^\mu (\op{T} Q - Q)$, matching the additional term introduced within the LK operator.
A crucial observation is that the fixed point of~\eqref{eq:u-bellman} is unique and can be estimated by standard TD methods under the same sampling conditions as the original value estimation problem.

Let $\rho \in \Delta(\set{S} \times \set{A})$ be a state-action reset (i.e., initial) distribution. 
At each iteration $k\in\N_0$, given a sample rollout $(s_{k,0}, a_{k,0}, r_{k,0}, \ldots, s_{k,n}, a_{k,n})$ with $(s_{k,0}, a_{k,0}) \sim \rho$ and $a_{k,i} \sim \mu(\cdot \vert s_{k,i})$ for $i \in \{1, \ldots, n\}$, we define the TD errors by
\begin{equation}
	\label{eq:td-error-q}
	\delta^{Q}_{k,j} \Let r_{k,j} + \gamma\,V(s_{k,j+1}) - Q(s_{k,j}, a_{k,j}), \quad j\in\{0,1,\ldots,n-1\},
\end{equation}
where $V(s) \Let \E_{a \sim \pi_Q}[Q(s, a)]$. 
The equation in~\eqref{eq:u-bellman} suggests the following TD error for $U$
\begin{equation}
	\label{eq:u-td-error}
	\delta^{U}_{k,j} \Let (1 - \tau)\,\delta^{Q}_{k,j} + \tau\,U(s_{k,j+1}, a_{k,j+1}) - U(s_{k,j}, a_{k,j}), \quad j\in\{0,1,\ldots,n-1\}.
\end{equation}
Intuitively, when $U$ updates on a faster timescale than $Q$, it remains approximately equilibrated to its $Q$-dependent fixed point throughout learning~\cite{perkins2012asynchronous}.
Furthermore, the TD estimate of $U$ admits a natural $n$-step and $\lambda$-return extension.
However, we defer this extension to the empirical studies in Section~\ref{sec:eval-continuous} and focus on the one-step estimator in our analysis.

For reset distribution $\rho$, behavior policy $\mu$, step sizes $\alpha_k,\beta_k>0$, and TD errors $\{\delta^{Q}_{k,j}, \delta^{U}_{k,j}\}_{j=0}^{n-1}$ defined along a sampled rollout by~\eqref{eq:td-error-q} and~\eqref{eq:u-td-error}, respectively, at each iteration $k\in\N_0$, we define the LKQL update at $(s, a) = (s_{k,0}, a_{k,0})$ by
\begin{equation}
	\label{eq:lkql-update}
    \textbf{LKQL}: \left\{
	\begin{array}{l}
		  Q_{k+1}(s, a) = Q_k(s, a) + \alpha_k \big(\ssum_{j=0}^{n-1} (\gamma\lambda)^j \delta^{Q}_{k,j} + \frac{(\gamma\lambda)^n}{1 - \gamma\lambda}\, U_k(s, a)\big), \\[4pt]
	    U_{k+1}(s, a) = U_k(s, a) + \beta_k\, \delta^{U}_{k,0}.
	\end{array}\right.
\end{equation}
Despite the spectral structure of the LK preconditioner introduced in Section~\ref{sec:spec-pre}, the update rule in~\eqref{eq:lkql-update} updates $Q$ at a single state-action pair per step, akin to Q-learning.
This contrasts with prior value-learning schemes that estimate the full preconditioner~\cite{devraj2017zap} or a low-rank approximation~\cite{lee2024deflated, kolarijani2025rank}, both of which update $Q$ at \emph{all} state-action pairs simultaneously.
As a result, LKQL readily extends to continuous state-action spaces, which we study in Section~\ref{sec:eval-continuous}.

While Theorem~\ref{thm:convergence-control} establishes operator-level convergence using the exact limiting kernel $\op{P}_\infty^\mu$, the LKQL update~\eqref{eq:lkql-update} replaces that with the auxiliary iterate $U$, estimated via TD learning with discount factor $\tau$.
It then remains to verify that this approximation and the resulting update scheme preserve convergence to $Q^\star$.
In a restricted setting with a fixed behavior policy $\mu$ and i.i.d.\ initial state-action pairs, the following theorem establishes an almost-sure convergence guarantee.
\begin{Thm}[Convergence of LKQL]
	\label{thm:lkql-value-learn}
	Consider the iterates $\{(Q_k, U_k)\}_{k \geq 0}$ of \emph{LKQL}~\eqref{eq:lkql-update} in a finite MDP with parameters $n\in\N$, $\lambda\in(0,1]$, $\tau\in[0,1)$, initialization $Q_0\in\set{Q}$, $U_0=0\in\set{Q}$, and TD errors \eqref{eq:td-error-q} and \eqref{eq:u-td-error} computed at each iteration $k\in\N_0$ along a rollout generated from reset distribution $\rho\in\Delta(\set{S}\times\set{A})$ and behavior policy $\mu\in\set{\Pi}$.
	Assume further that
	\begin{enumerate}
		\item the step sizes $\{\alpha_k\}, \{\beta_k\}$ are positive and satisfy $\ssum_k \alpha_k = \ssum_k \beta_k = \infty$, $\ssum_k(\alpha_k^2+\beta_k^2) < \infty$, and $\alpha_k/\beta_k \to 0$;
		\item the initial state-action pairs $z_k\Let(s_{k,0},a_{k,0})$ are sampled from $\rho$ independently of past samples, and $\rho(s,a) > 0$ for all $(s,a)$;
		\item $\lambda < \frac{1-\gamma}{2\gamma}$.
	\end{enumerate}
	Then, $Q_k \to Q^\star$ and $U_k \to 0$ almost surely.
\end{Thm}
Notably, the convergence guarantee holds for every fixed $\tau\in[0,1)$.
The \emph{coverage} assumption on $\rho$ guarantees that every state-action pair is visited infinitely often, the standard condition for asynchronous stochastic approximation~\cite{tsitsiklis1994asynchronous, borkar2023stochastic}.
By contrast, the i.i.d.\ assumption on $(s_{k,0}, a_{k,0})$ is an analytical idealization.
For instance, replay buffers in off-policy methods~\cite{fan2020theoretical, zhang2023dqn} and parallel environments in multi-actor methods~\cite{shen2023a3c} reduce temporal dependence and are commonly modeled as i.i.d.\ sampling oracles in finite-time analyses.
We use the latter setup, i.e., parallel sampling environments, in our experiments, which we discuss in Section~\ref{sec:evaluation}.
Related literature also accommodates Markovian sampling~\cite{bhandari2018finite, borkar2023stochastic, khodadadian2022federated}, showing that exact independence is not fundamental to this style of analysis.

Beyond these sampling assumptions, we follow prior off-policy value-learning analyses~\cite{melo2007convergence, chen2022finite, chen2023target, maei2010toward, devraj2017zap, lim2024regularized, lim2025understanding, wu2025unifying} and assume the behavior policy is fixed throughout the iterations.
Recent works relax this restriction by analyzing $Q$-dependent behavior policies~\cite{meyn2024projected, liu2025linear}, a direction we leave to future work.

\section{Evaluation}\label{sec:evaluation}

Section~\ref{sec:acc-ops} established convergence guarantees for the LK operator and the LKQL algorithm in finite MDPs.
We now assess how far these results carry from finite MDPs to those with large or continuous state spaces.
In Section~\ref{sec:eval-finite}, we compare LKHI($\pi$) against HI($\pi$) in operator norm and LKQL against truncated HQL (THQL) in value error for control, where the relevant operator norms and $Q^\star$ can be computed exactly.
In Section~\ref{sec:eval-continuous}, we demonstrate that LKQL, used as the target for a parameterized value function, is effective on the MuJoCo continuous-control benchmark~\cite{todorov2012mujoco}.

\subsection{Finite MDP}\label{sec:eval-finite}

\begin{figure}[t]
	\centering
	\includegraphics[width=0.78\linewidth]{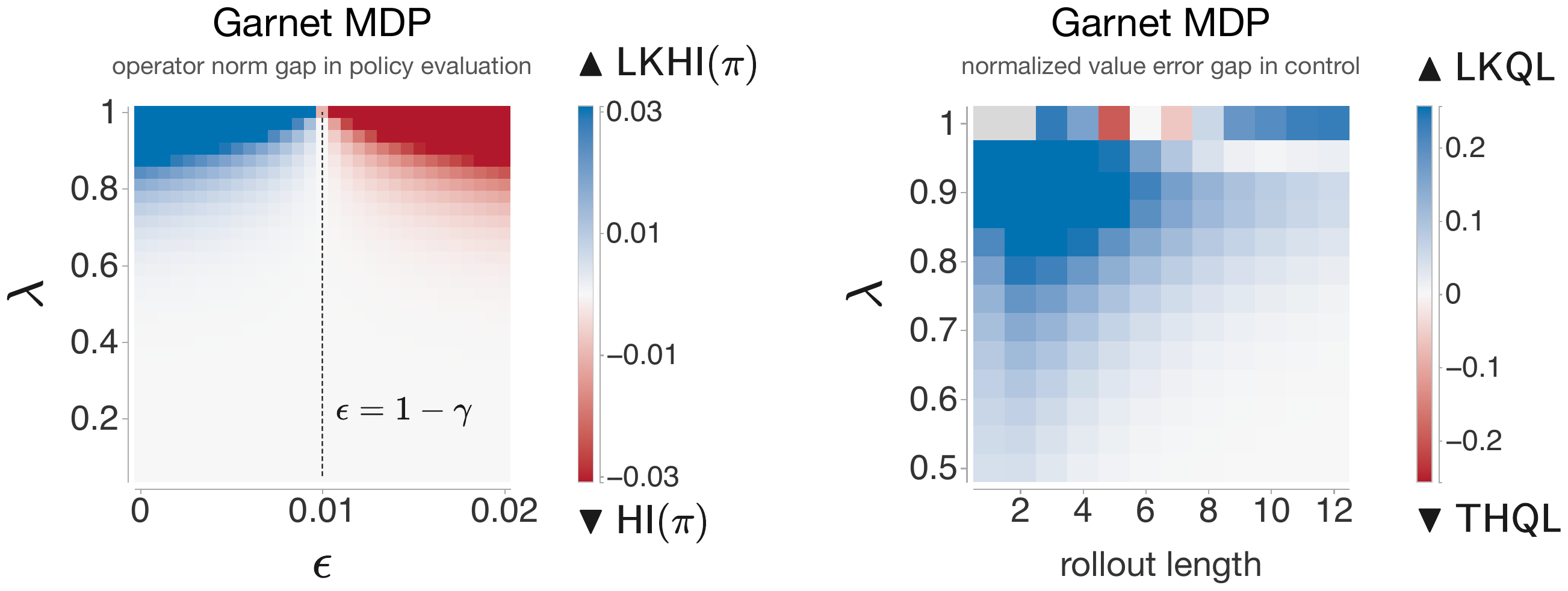}
	\caption{Garnet MDP results. Across both panels, blue marks the regime where the limiting-kernel method improves over its truncated counterpart.
	Left: Exact-model policy evaluation, showing the operator-norm difference between HI($\pi$) and LKHI($\pi$) over~$\epsilon$ and~$\lambda$, averaged over $15$ random Garnet MDPs (Theorem~\ref{thm:conv-lkhi-pi}).
	The dashed line $\epsilon=1-\gamma$ marks the theoretical boundary below which LKHI($\pi$) improves on HI($\pi$) for sufficiently large $n$.
	Right: Model-free control, showing the normalized value-error difference between truncated HQL (THQL) and LKQL after $T=8000$ iterations from a common initialization $Q_0$ (averaged over $32$ seeds), over rollout length~$n$ and~$\lambda$ (Theorem~\ref{thm:lkql-value-learn}).}
	\label{fig:finite}
\end{figure}

Under the exact model, Theorem~\ref{thm:conv-lkhi-pi} predicts that tail completion can tighten the policy evaluation rate of LKHI($\pi$)~\eqref{eq:LKH-iterations} relative to HI($\pi$)~\eqref{eq:HQL-truncated-def}.
We evaluate the $\infty$-norms of their error operators across 15 randomly generated Garnet MDPs with 50 states and 5 actions~\cite{archibald1995generation}.
For each MDP, we compute the difference
\begin{equation*}
	\label{eq:finite-norm-gap}
	\|\op{I} + \op{C}^\mu_{n,\lambda}(\gamma\op{P}^\pi - \op{I})\|_\infty 
	    - \|\op{I} + \wt{\op{C}}^\mu_{n,\lambda}(\gamma\op{P}^\pi - \op{I})\|_\infty.
\end{equation*}
Figure~\ref{fig:finite}~(left) shows the average of this difference across MDPs as $\epsilon$ and $\lambda$ vary, with fixed rollout length $n = 5$.
As shown, LKHI($\pi$) attains the lower rate at small~$\epsilon$ (near-on-policy) and large~$\lambda$, the regime Theorem~\ref{thm:conv-lkhi-pi} identifies.
The ranking reverses toward larger~$\epsilon$, where HI($\pi$) is tighter, and the two operators nearly coincide at small~$\lambda$.

In learning from samples, Theorem~\ref{thm:lkql-value-learn} guarantees that LKQL~\eqref{eq:lkql-update} converges almost surely to $Q^\star$ under a fixed behavior policy and i.i.d.\ initial state-action pairs.
Figure~\ref{fig:finite}~(right) shows, over rollout length~$n$ and~$\lambda$, the normalized value-error advantage of LKQL over THQL, given by
\begin{equation*}
	\label{eq:finite-value-gap}
	\bigl(\|Q_T - Q^\star\|_\infty - \|\wt{Q}_T - Q^\star\|_\infty\bigr)/\|Q_0 - Q^\star\|_\infty,
\end{equation*}
where $Q_T$ and $\wt{Q}_T$ are the final THQL and LKQL iterates, respectively, after $T=8000$ iterations from the same initialization $Q_0$ and using the same step sizes $\alpha_k$ for all $k$.
For $\lambda < 1$, LKQL attains the lower error at every rollout length. The gap is largest for large~$\lambda$ and short~$n$ and narrows as~$n$ grows, consistent with the diminishing contribution of the LK term.
Only at $\lambda = 1$ is LKQL not uniformly better across rollout lengths, diverging at $n \in \{1, 2\}$ and yielding a higher error than THQL at $n \in \{5, 6, 7\}$.

\subsection{Continuous control}\label{sec:eval-continuous}

We now extend LKQL to the continuous-control setting.
As discussed in Section~\ref{sec:val-learn}, many of the design choices in LKQL, including the choice of the limiting kernel for approximating the exact preconditioner and the TD estimator for the LK term, make such an extension possible.

To this end, we parameterize the value function, the LK term, and the policy and denote them by
\[
	V_\phi\colon \set{S} \to \R, \;\; U_\omega\colon \set{S} \to \R, \;\; \text{and} \;\; \pi_\theta\colon \set{S} \to \Delta(\set{A}),
\]
respectively, each modeled by a separate neural network.
The parameters are then updated using the corresponding loss functions, which we explain in detail later in this section.
The on- and off-policy estimators share the same sampling and parameter update scheme.
Algorithm~\ref{alg:lkql} summarizes the resulting procedure, which mirrors fitted Q-iteration~\cite{ernst2005tree, riedmiller2005neural} and actor-critic frameworks~\cite{mnih2016asynchronous, andrychowicz2021matters, cobbe2021phasic}.
That is, each outer iteration (over $k$) collects a single rollout from each of the $N$ parallel instances of the behavior policy, which we refer to interchangeably as ``actors'' (the behavior policy coincides with the target policy in on-policy learning).
An inner loop (over $i$) then performs $K$ minibatch updates on the resulting batch.
Parenthesized terms in Algorithm~\ref{alg:lkql} apply to the off-policy estimator, where ``and'' indicates additions and ``or'' indicates substitutions.

\begin{algorithm}[t]
	\caption{Actor-critic implementation of LKQL}
	\label{alg:lkql}
	\begin{algorithmic}[1]
		\STATE \texttt{initialize:} $V_\phi$ (and $P_\phi$), $U_\omega$, and $\pi_\theta$
		\FOR{iteration $k = 0, 1, 2, \dots$}
		\STATE \texttt{sample:} rollouts from $N$ parallel actors to form $\set{D}$
		\STATE \texttt{compute:} targets $G_V$ (or $G_Q$) and $G_U$ on $\set{D}$
		\FOR{iteration $i = 1, \dots, K$}
		\STATE \texttt{draw:} minibatch $\set{B}$ from $\set{D}$
		\STATE \texttt{update:} $\phi$, $\omega$, and $\theta$ on $\set{B}$
		\ENDFOR
		\ENDFOR
	\end{algorithmic}
\end{algorithm}

\textbf{On-policy estimator.}
Here, we use the actor-critic model of PPO~\cite{schulman2017proximal} as the backbone and focus on the value (critic) updates.
For the policy updates, we retain PPO's clipped surrogate objective.
The LKQL update~\eqref{eq:lkql-update} then translates to the loss functions
\begin{equation}
	\label{eq:loss-value-lk-on-policy}
	L(\phi) \Let \E_{s \sim \set{D}} \bigl[\bigl(G_V(s) - V_\phi(s)\bigr)^2\bigr],
	\;\; \text{and} \;\;
	L(\omega) \Let \E_{s\sim \set{D}} \bigl[\bigl(G_U(s) - U_\omega(s)\bigr)^2\bigr],
\end{equation}
where $\set{D}$ is the minibatch distribution of states sampled from the rollouts, and $G_V$ and $G_U$ are the respective target functions.
Crucially, the target functions are computed once before the inner loop and kept fixed throughout the parameter updates in the inner loop, hence reflecting the fitted Q-iteration scheme.

For notational simplicity, the definitions below assume that the truncation horizon $n$ equals the rollout length, although the two are separate hyperparameters in general (see Appendix~\ref{appx:exp-setup}).
At iteration $k$, for a rollout generated by an actor and starting from state $s_{k,0}$, we define the state-value estimate (i.e., the return, or target value) by
\begin{equation}
	\label{eq:G-V}
	G_V(s_{k,t}) \Let \ssum_{j=0}^{n-t-1} (\gamma\lambda)^j \delta^{V}_{k,t+j} + \frac{(\gamma\lambda)^{n-t}}{1 - \gamma\lambda}\, U_\omega(s_{k,t}) + V_\phi(s_{k,t}), \quad t\in\{ 0,1, \dots, n-1\},
\end{equation}
and the corresponding $V$-TD errors by
\begin{equation}
	\label{eq:td-error-on-policy}
	\delta^{V}_{k,t+j} \Let r_{k,t+j} + \gamma\,V_\phi(s_{k,t+j+1}) - V_\phi(s_{k,t+j}), \quad j\in\{ 0,1, \dots, n-t-1\}.
\end{equation}
Similarly, with a separate trace coefficient $\lambda_U \in (0, 1]$ for the LK term, we define the target for $U_\omega$ by
\[
	G_U(s_{k,t}) \Let \ssum_{j=0}^{n-t-1} (\tau\lambda_U)^{j} \delta^{U}_{k,t+j} + U_\omega(s_{k,t}), \quad t\in\{ 0,1, \dots, n-1\},
\]
and the corresponding $U$-TD errors, computed using the $V$-TD errors~\eqref{eq:td-error-on-policy}, by
\begin{equation}
    \label{eq:td-error-U}
	\delta^{U}_{k,t+j} \Let (1 - \tau)\,\delta^{V}_{k,t+j} + \tau\,U_\omega(s_{k,t+j+1}) - U_\omega(s_{k,t+j}), \quad j\in\{ 0,1, \dots, n-t-1\}.
\end{equation}
For the policy loss $L(\theta)$, we follow the clipped surrogate objective of PPO with an entropy regularization term, using the advantage term $A(s_{k,t}, a_{k,t}) \Let G_V(s_{k,t}) - V_\phi(s_{k,t})$.
Combining the policy loss with the loss functions in~\eqref{eq:loss-value-lk-on-policy}, we obtain the final loss function
\begin{equation}
	\label{eq:loss-on-policy}
	L(\phi, \omega, \theta) \Let L(\phi) + \beta_U L(\omega) + \beta_\pi L(\theta),
\end{equation}
where $\beta_U$ and $\beta_\pi$ are the respective loss coefficients.
Using $\pi_\theta$ for the $\texttt{sample}$ step and the loss function in~\eqref{eq:loss-on-policy} for the $\texttt{update}$ step, Algorithm~\ref{alg:lkql} summarizes the resulting on-policy LKQL extension of PPO.

\textbf{Off-policy estimator.}
We now turn to the off-policy LKQL estimator, in which we borrow the $Q$-value function architecture of the Normalized Advantage Function (NAF)~\cite{gu2016continuous}.
Specifically, we represent $Q$-values as the sum of the state value $V_\phi$ and a quadratic advantage centered at the deterministic target policy $\bar{\pi}_\theta\colon \set{S} \to \set{A}$, that is,
\[
	Q_{\theta, \phi}(s, a) \Let V_\phi(s) - \tfrac{1}{2}(a - \bar{\pi}_\theta(s))^\top P_\phi(s)(a - \bar{\pi}_\theta(s)),
\]
where $P_\phi(s)$ is a positive-definite matrix.
In this representation, with $\pi_\theta(\cdot \vert s)$ defined as the Dirac distribution at $\bar{\pi}_\theta(s)$, the relation $V_\phi(s) = \E_{a \sim \pi_\theta}[Q_{\theta, \phi}(s, a)]$ holds explicitly for any $\theta$ and $s \in \set{S}$, hence rendering the value learning scheme off-policy.

Due to the inherent relation $\bar{\pi}_\theta(s) = \argmax_a Q_{\theta, \phi}(s, a)$, the policy improvement step is implicit, and we define only the $Q$-value loss for updating the parameters $\phi$ and $\theta$ together by
\[
	L(\phi, \theta) \Let \E_{(s, a) \sim \set{D}} \bigl[\bigl(G_Q(s, a) - Q_{\theta, \phi}(s, a)\bigr)^2\bigr].
\]
Above, $\set{D}$ is again the minibatch distribution, and $G_Q$ is the return, given by
\begin{equation*}
	\label{eq:G-Q}
	G_Q(s_{k,t}, a_{k,t}) \Let \ssum_{j=0}^{n-t-1} (\gamma\lambda)^j \delta^{Q}_{k,t+j} + \frac{(\gamma\lambda)^{n-t}}{1 - \gamma\lambda}\, U_\omega(s_{k,t}) + Q_{\theta, \phi}(s_{k,t}, a_{k,t}), \quad t\in\{ 0,1, \dots, n-1\},
\end{equation*}
with $Q$-TD errors
\begin{equation}
	\label{eq:td-error-off-policy}
	\delta^{Q}_{k,t+j} \Let r_{k,t+j} + \gamma\,V_{\phi}(s_{k,t+j+1}) - Q_{\theta, \phi}(s_{k,t+j}, a_{k,t+j}), \quad j \in\{ 0, \dots, n-t-1\}.
\end{equation}
Although the value update is off-policy, the update for the LK term remains on-policy under the behavior policy $\mu$ used to define $\op{A}_\tau^\mu$.
We keep $U_\omega$ as a function of state.
The loss function $L(\omega)$, the target $G_U$, and the $U$-TD errors for the LK term are the same as in the on-policy estimator, except that the $Q$-TD errors~\eqref{eq:td-error-off-policy} replace the $V$-TD errors in computing the $U$-TD errors~\eqref{eq:td-error-U}.
The final loss function is then given by
\begin{equation}
	\label{eq:loss-off-policy}
	L(\phi, \omega, \theta) \Let L(\phi, \theta) + \beta_U L(\omega),
\end{equation}
where $\beta_U$ is the loss coefficient for the LK term.
The off-policy loss~\eqref{eq:loss-off-policy} drives the $\texttt{update}$ step of Algorithm~\ref{alg:lkql}, while the $\texttt{sample}$ step draws actions as $a = \bar{\pi}_\theta(s) + P_\phi(s)^{-1/2}w$ for $w \sim \mathcal{N}(0, I)$.

\begin{figure}[t]
	\centering
	\includegraphics[width=0.92\linewidth]{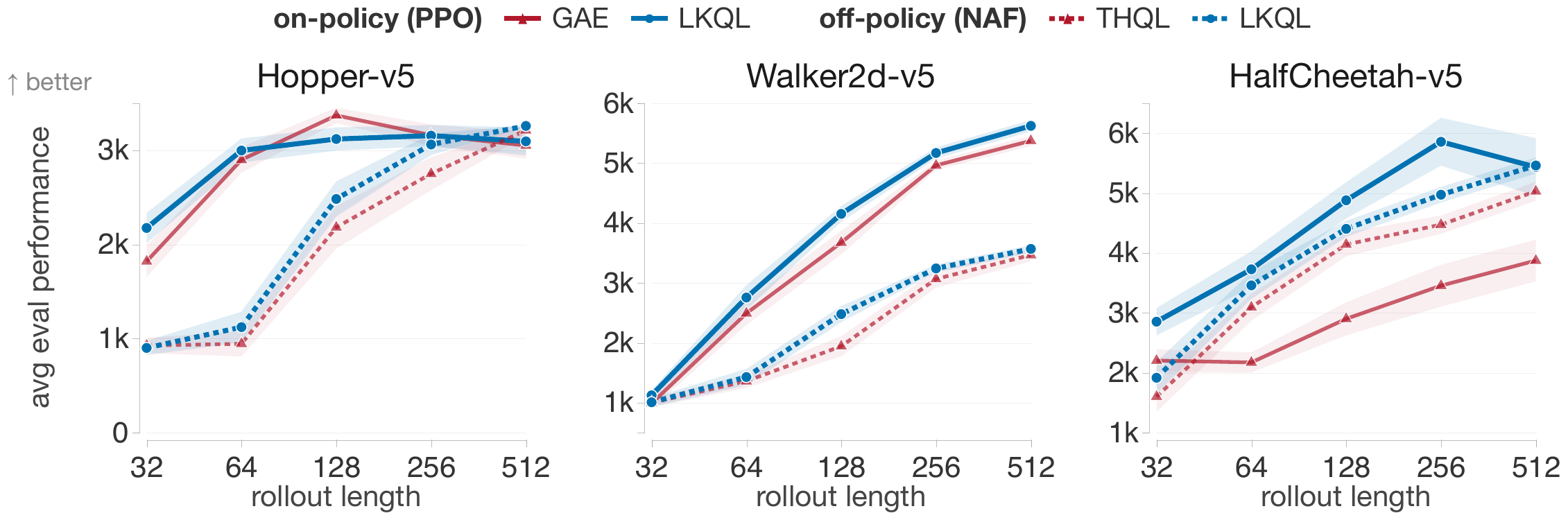}
	\caption{MuJoCo continuous-control results.
	Final evaluation performance (mean $\pm$ s.e.m.\ over $32$ seeds) versus rollout length, comparing LKQL (blue) against its baseline value estimator (red) within each actor-critic backbone: GAE in PPO (on-policy, solid) and truncated HQL (THQL) in NAF (off-policy, dashed).}
	\label{fig:mujoco-results}
\end{figure}

\textbf{Experiments.}
LKQL extends two standard $\lambda$-return baselines, namely, GAE~\cite{schulman2016gae} for on-policy learning and THQL~\cite{harutyunyan2016q} for off-policy learning.
Both are recovered from Algorithm~\ref{alg:lkql} by setting $U_\omega \equiv 0$.
Specifically, the baselines use the PPO backbone in the on-policy case and the NAF backbone in the off-policy case.
We benchmark LKQL against these baselines on \texttt{Hopper}, \texttt{Walker2d}, and \texttt{HalfCheetah} from the MuJoCo continuous-control benchmark~\cite{todorov2012mujoco}.
We train each method for $T=500$ outer iterations (see Algorithm~\ref{alg:lkql}) with $N=16$ parallel actors and several rollout lengths.
Figure~\ref{fig:mujoco-results} reports the final evaluation performance (cumulative episodic reward), averaged over $S=32$ seeds (runs).
Full hyperparameters and architectures are detailed in Appendix~\ref{appx:exp-setup}.
The baselines are tuned over their respective hyperparameter sets, and LKQL inherits the resulting best settings while tuning only its own hyperparameters, namely, the LK trace coefficient $\lambda_U$, the LK-term discount factor $\tau$, the $n$-step truncation horizon, and the optimizer parameters for $\omega$.

Figure~\ref{fig:mujoco-results} shows that backups with LK tail completion achieve higher returns than the baselines across most tested rollout lengths on \texttt{Hopper}, \texttt{Walker2d}, and \texttt{HalfCheetah} in both on- and off-policy learning.
The largest improvement occurs on \texttt{HalfCheetah}.
Unlike \texttt{Hopper} and \texttt{Walker2d}, which terminate early when the simulated body falls, \texttt{HalfCheetah} has no termination condition.
Its episodes run the full length, giving it the longest effective horizon.
This longer effective horizon in turn increases the relevance of the long-run behavior modeled by the LK term, which we argue accounts for the larger improvement here.
On this task, the improvement over the baseline is also more pronounced in the on-policy setting than in the off-policy setting, consistent with the acceleration being greatest near the on-policy regime.
See Remark~\ref{rem:conv-lkhi-pi} and Figure~\ref{fig:finite}.

\section{Conclusion}\label{sec:conc}

$\lambda$-return algorithms are often realized via $n$-step truncation~\cite{mnih2016asynchronous, schulman2016gae, hessel2018rainbow, espeholt2018impala}, which discards the tail of the non-truncated preconditioner and thereby yields a short-horizon update.
The LK Harutyunyan operator augments the $n$-step truncation by approximating the omitted tail using the limiting kernel (LK) of the behavior transition operator.
We showed that, for an aperiodic behavior chain and in the near-on-policy regime, the proposed operator improves the policy evaluation convergence rate over the truncated HQL (THQL) operator for sufficiently large $n$.

LKQL implements this operator through the double backup.
In practice, each backup maintains a separate value function, one for the critic and one for the LK term, allowing straightforward integration into existing actor-critic frameworks.
Together, the two backups bridge long and short horizons within a single value estimator.
We analyzed this estimator on two fronts, convergence and empirical performance.
We proved in Section~\ref{sec:val-learn} that it converges almost surely to the optimal value function $Q^\star$ in finite MDPs under a fixed behavior policy and i.i.d.\ initial state-action pairs.
Empirically, in Section~\ref{sec:eval-continuous}, we showed that, on the MuJoCo continuous-control benchmark, it outperforms GAE in the on-policy setting and THQL in the off-policy setting across most rollout lengths.

The LK extension of the truncated $\lambda$-return operators applies beyond THQL and GAE, our off-policy and on-policy baselines, respectively.
For instance, setting $\mu = \pi$, the HQL operator can be rewritten as
\begin{equation}
	\label{eq:HQL-pi}
	\op{R}_{\infty,\lambda}^{\pi,\pi} Q = \op{C}_{\infty,\lambda}^{\pi} \bigl(r + (1 - \lambda)\gamma \op{P}^\pi Q\bigr).
\end{equation}
However, replacing $\op{C}_{\infty,\lambda}^{\pi}$ with $\op{C}_{\infty,\lambda}^{\mu}$ in~\eqref{eq:HQL-pi} does not recover the HQL operator~$\op{R}_{\infty,\lambda}^{\mu,\pi}$.
Instead, the operator defined by $\op{S}_{\infty,\lambda}^{\mu,\pi} Q \Let \op{C}_{\infty,\lambda}^{\mu} (r + (1 - \lambda)\gamma \op{P}^\pi Q)$ recovers \textit{Peng \& Williams' Q($\lambda$)} (PQL)~\cite{peng1994incremental, kozuno2021revisiting}.
In a similar fashion, one can define the \emph{$n$-step Peng \& Williams operator} $\op{S}^{\mu,\pi}_{n,\lambda}\colon \set{Q} \to \set{Q}$ by
\begin{align*}
	\op{S}_{n,\lambda}^{\mu,\pi} Q & \Let Q + \op{C}_{n,\lambda}^{\mu} ( \op{T}^{\lambda\mu+(1-\lambda)\pi}Q - Q), \quad \mu,\pi\in\set{\Pi},\, n\in \N,\, \lambda \in (0,1].
\end{align*}
We note that, by changing the target policy, the limiting-kernel approximation readily extends to the truncated PQL operator $\op{S}^{\mu,\pi}_{n,\lambda}$.
Beyond PQL, conservative off-policy operators such as Retrace~\cite{munos2016safe}, with $c(s,a)=\min\{1,\nicefrac{\pi(a\vert s)}{\mu(a\vert s)}\}$, and Tree-Backup~\cite{precup2000eligibility}, with $c(s,a)=\pi(a\vert s)$, give rise to corrected operators $\op{P}^{c,\mu}$. Extending the double-backup scheme to these operators would require replacing $\op{P}^{\mu}$ in the preconditioner with $\op{P}^{c,\mu}$ and deriving a suitable tail approximation together with the corresponding auxiliary estimators. This nontrivial extension is left for future work.

Besides these potential extensions, we acknowledge several limitations of LKQL.
One such limitation is that our convergence analysis of LKQL in control assumes a fixed behavior policy.
We believe this assumption can be relaxed using the techniques developed by \citet{meyn2024projected} and \citet{liu2025linear}.
On the practical side, LKQL introduces a second critic for the LK term, modeled as a separate network with the same hidden sizes as the critic.
This increases both the memory footprint and the computational cost of the critic updates.
Additionally, estimating the LK term introduces the hyperparameters $\tau$ and $\lambda_U$, along with those of the optimizer used for its critic.

\appendix

\section{Technical Proofs}\label{appx:proofs}

In this section, we provide the proofs for the main results.
We begin by collecting notation and identities shared across all proofs.

\subsection{Notation and Preliminaries}\label{appx:notation}

Following \citet{harutyunyan2016q}, for a behavior policy $\mu$ and a target policy $\pi$, we define the decomposition
\begin{subequations}\label{eq:decomposeXW-def}
\begin{equation}\label{eq:decompose-def}
    \gamma \op{P}^\pi - \op{I} = \op{X} + \op{W},
\end{equation}
where
\begin{equation}\label{eq:XW-def}
    \op{X} \Let \gamma \lambda \op{P}^\mu - \op{I}, \qquad
	\op{W} \Let \gamma\lambda(\op{P}^\pi - \op{P}^\mu) + \gamma(1 - \lambda)\op{P}^\pi.
\end{equation}
\end{subequations}
Above, the operator $\op{X}$ captures the component aligned with the behavior chain (and hence the part of the Bellman residual absorbed by the preconditioner via the telescoping identity, as we shall see), while $\op{W}$ captures the off-policy mismatch and the $\lambda$-discounting.
We further define the scalar
\begin{equation*}\label{eq:xi-def}
	\xi \Let \frac{(\gamma \lambda)^{n}}{1 - \gamma \lambda}.
\end{equation*}
Note that $\xi$ is the $\infty$-norm of the LK tail completion operator $\frac{(\gamma\lambda)^n}{1-\gamma\lambda}\op{P}_\infty^\mu$ in~\eqref{eq:LK-precondition}.

The following identities are used repeatedly and referenced by name in subsequent proofs:
\begin{subequations}
    \begin{align}
    &\op{I} + \Cn\, \op{X} = (\gamma\lambda \op{P}^\mu)^n, & \text{(telescoping)} \label{eq:telescoping}\\
    &\op{P}_\infty^\mu \op{P}^\mu = \op{P}^\mu \op{P}_\infty^\mu = \op{P}_\infty^\mu, & \text{(absorption)} \label{eq:absorption}\\
    &\op{P}^\pi \e = \op{P}^\mu \e = \op{P}_\infty^\mu \e = \e. & \text{(row-sum)} \label{eq:row-sum}
\end{align}
\end{subequations}
We also note that, for any non-negative finite operator $\op{A}$, 
we have (see \citet[\S{}C.1]{puterman1994markov})
\begin{equation}
    \label{id:nonneg-norm}
    \norm{\op{A}}_\infty = \max_i \ssum_j \op{A}_{ij} = \norm{\op{A}\,\e}_\infty.
\end{equation}
In particular, using the row-sum identity~\eqref{eq:row-sum}, we have
\[
\|\op{P}^\pi\|_{\infty} = \|\op{P}^\mu\|_{\infty} = \|\op{P}_\infty^\mu\|_{\infty} = 1.
\]
Similarly, for the preconditioner $\op{C}_{n,\lambda}^{\mu}$, using~\eqref{id:nonneg-norm} and the row-sum identity~\eqref{eq:row-sum}, we obtain
\begin{equation}\label{eq:bound_trunc_precond}
    \|\op{C}_{n,\lambda}^{\mu}\|_{\infty} = \big\|  \ssum_{k = 0}^{n-1} (\lambda \gamma \op{P}^{\mu})^k\big\|_{\infty} = \ssum_{k = 0}^{n-1} (\lambda \gamma)^k = \frac{1-(\gamma\lambda)^n}{1-\gamma\lambda}.
\end{equation}
Also, by defining $\op{C}_{0, \lambda}^{\mu} \Let 0$, the construction of the preconditioner $\op{C}_{n,\lambda}^{\mu}$ satisfies the recursion
\begin{equation}\label{eq:trunc_precond_recur}
    \op{C}_{n+1, \lambda}^{\mu}  = \op{I} + \gamma\lambda \op{P}^\mu\, \Cn, \quad \forall n\in\N_0.
\end{equation}

Finally, recall that $\epsilon \Let \max_{s \in \set{S}}\norm{\mu(\cdot \vert s) - \pi(\cdot \vert s)}_1$ denotes the discrepancy between behavior and target policies.
We have (see the proof of \citet[Lem.~1]{harutyunyan2016q})
\begin{equation*}
    \label{eq:bound-transition-norm}
    \norm{\op{P}^\pi - \op{P}^\mu}_\infty \leq \epsilon.
\end{equation*}
Using the preceding inequality and the definition of $\op{W}$ in~\eqref{eq:XW-def}, we can also write
\begin{equation}\label{eq:bound_W}
    \norm{\op{W}}_\infty \leq \gamma\lambda\epsilon + \gamma(1-\lambda) = (1-\gamma\lambda)\eta_{\epsilon}.
\end{equation}

\subsection{Proof of Lemma~\ref{lem:convergence-h}}\label{proof:convergence-h}

\textbf{Policy evaluation.}
Let $E_\ell \Let Q_\ell - Q^\pi$ be the value error of the $\ell$-th iteration of HI($\pi$). 
Since $\op{T}^\pi Q = \gamma \op{P}^\pi Q + r$ and $\op{T}^\pi Q^\pi = Q^\pi$, the Bellman residual factors as
\[
\op{T}^\pi Q - Q = (\op{T}^\pi Q - \op{T}^\pi Q^\pi) - (Q- Q^\pi) = (\gamma \op{P}^\pi - \op{I})(Q - Q^\pi).
\]
Then, using the definition of $\op{R}^{\mu, \pi}_{n, \lambda}$, we can write $E_{\ell+1} = \op{M}_n\, E_\ell$, where $\op{M}_n \Let \op{I} + \Cn(\gamma \op{P}^\pi - \op{I})$ is the error operator of HI($\pi$). 
Below, we show that $m_n \Let \|\op{M}_n\|_\infty \leq \alpha_{\mathrm{e}}$, from which $\|E_\ell\|_{\infty} \leq \alpha_{\mathrm{e}}^{\ell}\, \|E_0\|_{\infty}$ follows.

First, using the decomposition~\eqref{eq:decomposeXW-def} and applying the telescoping identity~\eqref{eq:telescoping}, the error operator can be equivalently characterized as $\op{M}_n = (\gamma\lambda \op{P}^\mu)^n + \Cn \op{W}$.
Then, the recursion~\eqref{eq:trunc_precond_recur} for the preconditioner implies that $\op{M}_n$ satisfies
\begin{equation*}
	\op{M}_{n+1} = \op{W} + \gamma\lambda \op{P}^\mu\, \op{M}_{n}, \quad \forall n\in\N.
\end{equation*}
Taking norms in the preceding recursion and using~\eqref{eq:bound_W} then gives
\begin{equation}\label{eq:hi_error_norm}
	m_{n+1} \leq \norm{\op{W}}_\infty + \gamma\lambda\,m_{n} \leq (1-\gamma\lambda)\,\eta_{\epsilon} + \gamma\lambda\,m_{n}, \quad \forall n\in\N.
\end{equation}
For $n = 1$, we have 
\begin{equation*}
	m_1 = \norm{\op{M}_1}_\infty = \norm{\op{W} + \gamma\lambda\op{P}^\mu}_\infty = \norm{\gamma\op{P}^\pi}_\infty = \gamma = \alpha_{\mathrm{e}}.
\end{equation*}
For $n \geq 2$, iterating~\eqref{eq:hi_error_norm} also yields
\begin{align*}
	m_n \leq (\gamma\lambda)^{n-1}m_1 + (1-\gamma\lambda)\eta_{\epsilon}\ssum_{j=0}^{n-2}(\gamma\lambda)^j 
	  = (\gamma\lambda)^{n-1}\bigl(\gamma - \eta_{\epsilon}\bigr) + \eta_{\epsilon} = \alpha_{\mathrm{e}}.
\end{align*}
This completes the first part of the proof. 
For $n\geq2$, we now consider the condition $\alpha_{\mathrm{e}} < 1$ to derive the bound on~$\epsilon$.
Starting from the definition of $\alpha_{\mathrm{e}}$ and isolating $\eta_{\epsilon}$
yields
\begin{equation*}
	\alpha_{\mathrm{e}} = (\gamma\lambda)^{n-1}(\gamma - \eta_{\epsilon}) + \eta_{\epsilon} < 1
	\;\iff\;
	\eta_{\epsilon}\bigl(1 - (\gamma\lambda)^{n-1}\bigr) < 1 - \gamma(\gamma\lambda)^{n-1}.
\end{equation*}
Substituting $\eta_{\epsilon}$ then gives
\begin{equation*}
	\gamma\bigl(1 + \lambda(\epsilon - 1)\bigr)\bigl(1 - (\gamma\lambda)^{n-1}\bigr) < (1 - \gamma\lambda)\bigl(1 - \gamma(\gamma\lambda)^{n-1}\bigr).
\end{equation*}
Rearranging terms yields $\epsilon < \epsilon_{\max}$ as claimed. 

\textbf{Control.}
Now, let $E_\ell \Let Q_\ell - Q^\star$ be the value error of the $\ell$-th iteration of HI for the control problem. 
Below, we show that $\|E_{\ell+1}\|_{\infty} \leq \beta_{\mathrm{c}}\, \|E_{\ell}\|_{\infty}$, from which the first claim in the statement of the lemma follows. 
Recall that $\pi_Q$ is greedy with respect to $Q$, and hence $\op{T}^{\pi_Q} Q = \op{T} Q$.
Using $\op{T} Q^\star = Q^\star$ and the recursion~\eqref{eq:trunc_precond_recur} for the preconditioner (with $\op{C}_{0, \lambda}^{\mu} \Let 0$), the value error can be written as
\begin{align*}
	E_{\ell+1} &=  \op{R}^{\mu, \pi_Q}_{n, \lambda} Q_\ell - Q^\star 
    = Q_\ell + \op{C}_{n,\lambda}^{\mu} (\op{T} Q_\ell - Q_\ell) - Q^\star \\
    &=  \op{C}_{n,\lambda}^{\mu} \big( (\op{T} Q_\ell - \op{T} Q^\star) - ( Q_\ell - Q^\star) \big) + ( Q_\ell - Q^\star) \\
    &= \Cn(\op{T} Q_\ell - \op{T} Q^\star) - \gamma\lambda\op{P}^\mu \op{C}_{n-1, \lambda}^{\mu}E_\ell.
\end{align*}
Taking norms and using the fact that $\op{T}$ is a $\gamma$-contraction, we have 
\begin{align*}
	\|E_{\ell+1}\|_{\infty} \leq \Big( \gamma\,\|\Cn\|_{\infty} + \gamma\lambda\, \|\op{P}^\mu\|_{\infty}\, \|\op{C}_{n-1, \lambda}^{\mu}\|_{\infty} \Big) \|E_{\ell}\|_{\infty}.
\end{align*}
Finally, using the bound~\eqref{eq:bound_trunc_precond} (which also holds for $n=0$ since $\op{C}_{0, \lambda}^{\mu} \Let 0$ by our definition), we obtain
\begin{align*}
	\|E_{\ell+1}\|_{\infty} \leq \frac{\gamma(1 - (\gamma\lambda)^n) + \gamma\lambda - (\gamma\lambda)^n}{1 - \gamma\lambda}\|E_{\ell}\|_{\infty}
	  = \biggl[\eta_2 - \frac{(1 + \gamma)(\gamma\lambda)^n}{1 - \gamma\lambda}\biggr] \|E_{\ell}\|_{\infty} = \beta_{\mathrm{c}}\, \|E_{\ell}\|_{\infty}.
\end{align*}
We next consider the condition $\beta_{\mathrm{c}} < 1$ to derive the corresponding inequality for $\lambda$.
The first characterization of $\beta_{\mathrm{c}}$ above yields
\begin{equation*}
	\beta_{\mathrm{c}} < 1 \iff \gamma(1+\lambda) - (1+\gamma)(\gamma\lambda)^n < 1 - \gamma\lambda.
\end{equation*}
Rearranging then gives the implicit inequality
\begin{equation*}
	\lambda < \frac{1 - \gamma}{2\gamma} + \frac{(1+\gamma)(\gamma\lambda)^n}{2\gamma}.
\end{equation*}
This completes the proof.

\subsection{Proof of Lemma~\ref{lem:lk-aperiodicity}}\label{proof:lk-aperiodicity}

Let $\Delta^{\pi}_k \Let (\op{P}^\pi)^k - \op{P}_\infty^\pi$ for $k\in\N_0$, so that the finite-horizon limiting-kernel residual is given by $\delta^{\pi}_k = \norm{\Delta^{\pi}_k}_\infty$. 
The residuals from truncation and tail completion are given by
\begin{align*}
	\delta_n      & \Let \|\op{C}^{\pi}_{\infty,\lambda} - \op{C}^{\pi}_{n,\lambda}\|_\infty = \big\|\ssum_{k \geq n} (\gamma \lambda)^k (\op{P}^\pi)^k\big\|_\infty,                                     \\
	\wt{\delta}_n & \Let \|\op{C}^{\pi}_{\infty,\lambda} - \wt{\op{C}}^{\pi}_{n,\lambda}\|_\infty = \big\|\ssum_{k \geq n} (\gamma \lambda)^k \Delta^{\pi}_k\big\|_\infty,
\end{align*}
where the second equality follows from the representation~\eqref{eq:LK-precondition-alternative}.
Now, observe that, by the non-negative operator norm identity~\eqref{id:nonneg-norm} and the row-sum identity~\eqref{eq:row-sum}, we obtain
\begin{equation*}
	\delta_n = \big\|\ssum_{k \geq n} (\gamma \lambda)^k (\op{P}^\pi)^k \e\big\|_\infty 
    = \big\|\ssum_{k \geq n} (\gamma \lambda)^k \e\big\|_\infty 
    = \ssum_{k \geq n} (\gamma \lambda)^k
    = \dfrac{(\gamma\lambda)^n}{1 - \gamma\lambda}.
\end{equation*}
For $\wt{\delta}_n$, the triangle inequality and the aperiodicity bound~\eqref{eq:p-infty-convergence} give
\begin{equation*}
	\wt{\delta}_n \leq \ssum_{k \geq n} (\gamma \lambda)^k\,\delta^{\pi}_k 
    \leq  c_{\pi} \ssum_{k \geq n} (\gamma \lambda \sigma_{\pi})^k = \dfrac{c_{\pi}\,(\gamma\lambda\sigma_{\pi})^n}{1 - \gamma\lambda\sigma_{\pi}}.
\end{equation*}
Dividing the bound on $\wt{\delta}_n$ by the expression for $\delta_n$, we obtain
\begin{equation*}
	\frac{\wt{\delta}_n}{\delta_n} \leq \frac{c_{\pi}(1 - \gamma\lambda)}{1 - \gamma\lambda\sigma_{\pi}} \cdot \sigma_{\pi}^n  =  \mathcal{O}(\sigma_{\pi}^n).
\end{equation*}

It remains to show the strict inequality for $n \geq n_\pi$. The absorption identity~\eqref{eq:absorption} gives $\Delta^{\pi}_{k+1} = \Delta^{\pi}_k\,\op{P}^\pi$, and hence 
\begin{equation*}
	\delta^{\pi}_{k+1} \leq \delta^{\pi}_k\, \norm{\op{P}^\pi}_\infty = \delta^{\pi}_k.
\end{equation*}
That is, $\delta^{\pi}_k$ is non-increasing in $k$.
By the definition of $n_\pi$ in~\eqref{eq:tail-improvement-threshold}, the residual $\delta^{\pi}_k$ drops below one at $k = n_\pi$, and by monotonicity it stays below one for every $k \geq n_\pi$.
Substituting this strict bound gives
\begin{equation*}
	\wt{\delta}_n \leq \ssum_{k \geq n} (\gamma\lambda)^k \, \delta^{\pi}_k < \ssum_{k \geq n} (\gamma\lambda)^k = \delta_n, \quad \forall n\geq n_\pi.
\end{equation*}

\subsection{Proof of Theorem~\ref{thm:conv-lkhi-pi}}\label{proof:conv-lkhi-pi}

Let $E_\ell \Let Q_\ell - Q^\pi$ be the value error of the $\ell$-th iteration of LKHI($\pi$). 
Using the same argument as in the policy evaluation part of the proof of Lemma~\ref{lem:convergence-h}, we obtain $E_{\ell+1} = \wt{\op{M}}_n\, E_\ell$, where $\wt{\op{M}}_n \Let \op{I} + \wt{\op{C}}^{\mu}_{n,\lambda}(\gamma\op{P}^\pi - \op{I})$ is the error operator.
Below, we show that $\|\wt{\op{M}}_n\|_\infty \leq \wt{\alpha}_{\mathrm{e}}$, from which the first claim in the statement of the theorem follows.

Combining the identity $\wt{\op{C}}^{\mu}_{n, \lambda} = \op{C}^{\mu}_{n, \lambda} + \xi \op{P}_\infty^\mu$ in~\eqref{eq:LK-precondition} with the decomposition~\eqref{eq:decomposeXW-def}, we can expand the error operator into four terms as follows
\begin{equation}\label{eq:lkhi-pe-expanded}
	\wt{\op{M}}_n = (\op{I} + \Cn\op{X}) + \Cn\op{W} + \xi\op{P}_\infty^\mu\op{X} + \xi\op{P}_\infty^\mu\op{W}.
\end{equation}
Now, observe that the telescoping identity~\eqref{eq:telescoping} collapses the first term to $(\gamma\lambda\op{P}^\mu)^n$, while the absorption identity~\eqref{eq:absorption} reduces the third term to $-(\gamma\lambda)^n\op{P}_\infty^\mu$.
Summing these two terms exposes the mixing residual $\Delta^{\mu}_n \Let (\op{P}^\mu)^n - \op{P}_\infty^\mu$, and~\eqref{eq:lkhi-pe-expanded} reduces to
\begin{equation}\label{eq:lkhi-pe-error-operator}
	\wt{\op{M}}_n = (\gamma\lambda)^n \Delta^{\mu}_n + \Cn\op{W} + \xi\op{P}_\infty^\mu\op{W}.
\end{equation}
Applying the triangle inequality and using $\delta^{\mu}_n = \|\Delta^{\mu}_n\|_{\infty}$ by definition and the bounds~\eqref{eq:bound_trunc_precond} and~\eqref{eq:bound_W}, we obtain
\begin{align*}
	\|\wt{\op{M}}_n\|_\infty & \leq  (\gamma\lambda)^n \delta^{\mu}_n + \dfrac{1-(\gamma\lambda)^n}{1-\gamma\lambda} \cdot (1-\gamma\lambda)\eta_{\epsilon} + \dfrac{(\gamma \lambda)^{n}}{1 - \gamma \lambda} \cdot (1-\gamma\lambda)\eta_{\epsilon}\\
    & = (\gamma\lambda)^n \delta^{\mu}_n + \eta_{\epsilon} = \wt\alpha_{\mathrm{e}}.
\end{align*}

For convergence, we require $\wt\alpha_{\mathrm{e}} < 1$, which is equivalent to $\eta_{\epsilon} < 1 - (\gamma\lambda)^n \delta^{\mu}_n$.
Substituting $\eta_{\epsilon}$, we arrive at 
\begin{equation*}
	\gamma(1-\lambda) + \gamma\lambda\,\epsilon < (1-\gamma\lambda)\bigl(1 - (\gamma\lambda)^n \delta^{\mu}_n\bigr),
\end{equation*}
which rearranges to $\epsilon < \wt{\epsilon}_{\max}$ as claimed. 

Finally, comparing against the HI($\pi$) rate $\alpha_{\mathrm{e}}$ from Lemma~\ref{lem:convergence-h}, we have
\begin{equation*}
	\alpha_{\mathrm{e}} - \wt\alpha_{\mathrm{e}}
	= (\gamma\lambda)^{n-1}(\gamma - \eta_{\epsilon}) - (\gamma\lambda)^n \delta^{\mu}_n
	= (\gamma\lambda)^{n-1}\bigl((\gamma - \eta_{\epsilon}) - \gamma\lambda\delta^{\mu}_n\bigr).
\end{equation*}
Using the identity $\gamma - \eta_{\epsilon} = \gamma\lambda(1-\gamma-\epsilon)/(1-\gamma\lambda)$, we obtain
\begin{equation*}
	\wt\alpha_{\mathrm{e}} < \alpha_{\mathrm{e}} \iff  \delta^{\mu}_n < \frac{1-\gamma-\epsilon}{1-\gamma\lambda}.
\end{equation*}
This completes the proof. 

\subsection{Proof of Corollary~\ref{cor:conv-lkhi-pi}}\label{proof:cor:conv-lkhi-pi}
The result follows from Theorem~\ref{thm:conv-lkhi-pi} and the aperiodicity bound~\eqref{eq:p-infty-convergence}. 
Indeed, if 
\[
n > \dfrac{\log\left(\frac{c_\mu(1-\gamma\lambda)}{1-\gamma-\epsilon}\right)}{\log(1/\sigma_\mu)} = \dfrac{\log\left(\frac{1-\gamma-\epsilon}{c_\mu(1-\gamma\lambda)}\right)}{\log(\sigma_\mu)},
\]
then
\[
\log\left(\dfrac{1-\gamma-\epsilon}{c_\mu(1-\gamma\lambda)}\right) > n\log(\sigma_\mu),
\]
and we have 
\[
\delta^{\mu}_n \leq c_{\mu}\sigma_{\mu}^n < \frac{1-\gamma-\epsilon}{1-\gamma\lambda}.
\]

\subsection{Proof of Theorem~\ref{thm:convergence-control}}\label{proof:convergence-control}

Let $E_\ell \Let Q_\ell - Q^\star$ be the value error of the $\ell$-th iteration of LKHI for the control problem. 
Below, we show that $\|E_{\ell+1}\|_{\infty} \leq \eta_2\, \|E_{\ell}\|_{\infty}$, from which the first claim in the statement of the theorem follows. 
Recall that $\pi_Q$ is greedy with respect to $Q$, and hence $\op{T}^{\pi_Q} Q = \op{T} Q$.
Then, using the fact that $\op{T} Q^\star = Q^\star$, we have
\begin{align*}
	\|E_{\ell+1}\|_{\infty} &=  \big\|\wt{\op{R}}^{\mu, \pi_Q}_{n, \lambda} Q_\ell - Q^\star\big\|_{\infty} 
    = \big\|Q_\ell + \wt{\op{C}}_{n,\lambda}^{\mu} (\op{T} Q_\ell - Q_\ell) - Q^\star\big\|_{\infty}  \\
    &=  \big\|\wt{\op{C}}_{n,\lambda}^{\mu} \big( (\op{T} Q_\ell - \op{T} Q^\star) - ( Q_\ell - Q^\star) \big) + ( Q_\ell - Q^\star)\big\|_{\infty}  \\
    &= \big\|(\op{I} - \wt{\op{C}}^{\mu}_{n,\lambda})(Q_\ell - Q^\star) + \wt{\op{C}}^{\mu}_{n,\lambda}(\op{T} Q_\ell - \op{T} Q^\star)\big\|_{\infty} \\
    &\leq  \|\wt{\op{C}}^{\mu}_{n,\lambda} - \op{I} \|_{\infty}\,  \|E_{\ell}\|_{\infty} + \|\wt{\op{C}}^{\mu}_{n,\lambda}\|_{\infty}\, \|\op{T} Q_\ell - \op{T} Q^\star\|_{\infty}.
\end{align*}
Then, using the fact that $\op{T}$ is a $\gamma$-contraction, we arrive at
\begin{align}\label{eq:error_dyn_LKHI_cont}
	\|E_{\ell+1}\|_{\infty} 
    &\leq \Big( \| \wt{\op{C}}^{\mu}_{n,\lambda} - \op{I} \|_{\infty} + \gamma\, \|\wt{\op{C}}^{\mu}_{n,\lambda}\|_{\infty} \Big)\, \|E_{\ell}\|_{\infty}.
\end{align}
Thus, the problem reduces to computing the $\infty$-norms of $\wt{\op{C}}^{\mu}_{n,\lambda}$ and $\wt{\op{C}}^{\mu}_{n,\lambda} - \op{I}$. 
Observe that both operators are non-negative.
In particular, using the row-sum identity~\eqref{eq:row-sum}, we have
\begin{equation*}
    \wt{\op{C}}^{\mu}_{n,\lambda}\e = \big(\ssum_{k = 0}^{n-1} (\lambda \gamma \op{P}^{\mu})^k + \xi \op{P}_\infty^\mu  \big)\e 
    = \big(\ssum_{k = 0}^{n-1} (\lambda \gamma)^k + \xi  \big)\e
    = \dfrac{1}{1-\gamma\lambda} \e.
\end{equation*}
Similarly, 
\begin{equation*}
    (\wt{\op{C}}^{\mu}_{n,\lambda}-\op{I})\e = \big(\ssum_{k = 1}^{n-1} (\lambda \gamma \op{P}^{\mu})^k + \xi \op{P}_\infty^\mu  \big)\e 
    = \big(\ssum_{k = 1}^{n-1} (\lambda \gamma)^k + \xi  \big)\e
	    = \dfrac{\gamma\lambda}{1-\gamma\lambda} \e,
\end{equation*}
where the sum is empty when $n=1$. 
Then, the non-negative operator norm identity~\eqref{id:nonneg-norm} implies that $\|\wt{\op{C}}^{\mu}_{n,\lambda}\|_\infty = 1/(1-\gamma\lambda)$ and 
$\|\wt{\op{C}}^{\mu}_{n,\lambda} -\op{I}\|_\infty = \gamma\lambda/(1-\gamma\lambda)$. 
Substituting these norms into~\eqref{eq:error_dyn_LKHI_cont}, we arrive at
\begin{equation*}
	\|E_{\ell+1}\|_{\infty} \leq \Big(\frac{\gamma\lambda}{1-\gamma\lambda} + \frac{\gamma}{1-\gamma\lambda}\Big)\, \|E_{\ell}\|_{\infty}
	= \eta_2\,\|E_{\ell}\|_{\infty}.
\end{equation*}
Finally, requiring $\eta_2 < 1$ for convergence rearranges to $\lambda < (1-\gamma)/(2\gamma)$, which completes the proof.

\subsection{Proof of Theorem~\ref{thm:lkql-value-learn}}\label{proof:lkql-value-learn}

We use the standard two-timescale stochastic approximation theorem based on the ODE method~\cite[\S8.1]{borkar2023stochastic}.
For completeness, we state the theorem below.

\begin{Thm}[Two-timescale SA with ODE limit]
	\label{thm:ttsa-ode}
    \emph{\cite[\S8.1]{borkar2023stochastic}}
	Let $\{\set{F}_k\}_{k\geq0}$ be a filtration, and consider iterates $(X_k,Y_k)\in\R^d\times\R^m$ satisfying
	\begin{align*}
		X_{k+1} & = X_k + \alpha_k\bigl(H(X_k,Y_k)+M^X_{k+1}\bigr), \\
		Y_{k+1} & = Y_k + \beta_k\bigl(G(X_k,Y_k)+M^Y_{k+1}\bigr).
	\end{align*}
	Assume:
	\begin{enumerate}[label=(\roman*)]
		\item The step sizes are positive and satisfy
		      \begin{equation*}
			      \ssum_k \alpha_k = \ssum_k \beta_k = \infty, \qquad
			      \ssum_k(\alpha_k^2 + \beta_k^2) < \infty, \qquad
			      \alpha_k/\beta_k \to 0.
		      \end{equation*}
		\item The maps $H$ and $G$ are Lipschitz.
		\item $\{M^X_{k+1}\}$ and $\{M^Y_{k+1}\}$ are martingale-difference noise terms with respect to $\{\set{F}_k\}$ and, for some $K<\infty$,
		      \begin{equation*}
			      \E\bigl[\norm{M^X_{k+1}}^2 + \norm{M^Y_{k+1}}^2 \mid \set{F}_k\bigr]
			      \leq K\bigl(1+\norm{X_k}^2+\norm{Y_k}^2\bigr).
		      \end{equation*}
		\item The iterates are almost surely bounded, that is, $\sup_k(\norm{X_k}+\norm{Y_k})<\infty$.
		\item For every fixed $X$, the ODE $\dot{Y}(t)=G(X,Y(t))$ has a unique globally asymptotically stable equilibrium $Y^\star(X)$, and $Y^\star$ is Lipschitz.
		\item The ODE $\dot{X}(t)=H(X(t),Y^\star(X(t)))$ has a unique globally asymptotically stable equilibrium $X^\star$.
	\end{enumerate}
	Then $X_k\to X^\star$ and $Y_k\to Y^\star(X^\star)$ almost surely.
\end{Thm}

We now apply Theorem~\ref{thm:ttsa-ode} with $X_k=Q_k$ and $Y_k=U_k$.
Let $\{\set{F}_k\}_{k\geq0}$ be the filtration generated by $(Q_0,U_0)$ and the samples collected before iteration~$k$.
Then $Q_k$ and $U_k$ are $\set{F}_k$-measurable.
Define the drift functions $g,h:\set{Q}\times\set{Q}\to\set{Q}$ for each state-action pair $z\Let(s,a)$ by
\begin{align*}
	g(Q,U)(z)
	 & \Let (1-\tau)(\op{T}Q-Q)(z)+\tau\op{P}^{\mu}U(z)-U(z),          \\
	h(Q,U)(z)
	 & \Let \bigl(\op{C}_{n,\lambda}^{\mu}(\op{T}Q-Q)\bigr)(z)+\xi U(z).
\end{align*}

Let $\op{B}_\rho:\set{Q}\to\set{Q}$ denote the positive diagonal operator with $\op{B}_\rho X(z) \Let \rho(z)X(z)$, where $\rho$ is the reset distribution.
By Assumption~2 in Theorem~\ref{thm:lkql-value-learn}, the initial state-action pair $z_k\Let(s_{k,0},a_{k,0})$ is independent of $\set{F}_k$ and has distribution $\rho$.
The coverage assumption on $\rho$ implies that the diagonal entries of $\op{B}_\rho$ lie in $(0,1]$.
Thus, the expected drifts of the two-timescale recursion are $H\Let\op{B}_\rho h$ and $G\Let\op{B}_\rho g$.

The following two lemmas establish global asymptotic stability for the equilibria of the fast and slow ODEs, which we use below to verify assumptions (iv)--(vi).
\begin{Lem}[Fast ODE global asymptotic stability]
	\label{lem:fast-gas}
	For any $V\in\set{Q}$ and diagonal operator $\op{B}:\set{Q}\to\set{Q}$ with diagonal entries in $(0,1]$, the ODE
	\begin{equation*}
		\dot{U}(t)=\op{B}\bigl(V+\tau\op{P}^{\mu}U(t)-U(t)\bigr)
	\end{equation*}
	has the unique globally asymptotically stable equilibrium $U^\dagger=(\op{I}-\tau\op{P}^{\mu})^{-1}V$.
\end{Lem}

\begin{proof}
	Recall that $\tau\in[0,1)$ and $\op{P}^{\mu}$ is stochastic.
	It follows that, for fixed $V$, the map $U \mapsto V+\tau\op{P}^{\mu}U$ is a $\tau$-contraction in $\infty$-norm with fixed point $U^\dagger$.
	Applying the argument in \citet[\S6.4]{borkar2023stochastic} for diagonally scaled ODEs, we conclude that $U^\dagger$ is the unique globally asymptotically stable equilibrium.
\end{proof}

\begin{Lem}[Slow ODE global asymptotic stability]
	\label{lem:slow-gas}
	Let $\widehat{\op{T}}\colon \set{Q} \to \set{Q}$ be $\gamma$-Lipschitz in $\infty$-norm with fixed point $Q^\dagger$, and $\op{B}:\set{Q}\to\set{Q}$ be a diagonal operator with diagonal entries in $(0,1]$.
	If $\lambda < (1-\gamma)/(2\gamma)$, the ODE
	\begin{equation*}
		\dot{Q}(t) = \op{B}\,\bigl(\op{C}_{n,\lambda}^\mu + \xi\op{A}_\tau^\mu\bigr)\bigl(\widehat{\op{T}}Q(t) - Q(t)\bigr)
	\end{equation*}
	has $Q^\dagger$ as the unique globally asymptotically stable equilibrium.
\end{Lem}
\begin{proof}
	The map $Q \mapsto Q + (\op{C}_{n,\lambda}^\mu + \xi\op{A}_\tau^\mu)(\widehat{\op{T}}Q - Q)$ is an $\eta_2$-contraction in $\infty$-norm with fixed point $Q^\dagger$, using the same arguments as in the proof of Theorem~\ref{thm:convergence-control}.
	In particular, observe that the operator $\op{A}_\tau^\mu$ shares the same norm identities used there for $\op{P}_\infty^\mu$.
	Furthermore, by \citet[Thm.~4.2]{borkar1997analog}, for any $\infty$-norm contraction $\mathcal{F}$ with fixed point $X^\star$ and any scalar $a > 0$, the ODE $\dot{X} = a(\mathcal{F}X - X)$ has $X^\star$ as its unique globally asymptotically stable equilibrium.
	Combining this with the argument in \citet[\S6.4]{borkar2023stochastic} for diagonally scaled ODEs, we conclude that $Q^\dagger$ is globally asymptotically stable.
\end{proof}

In what follows, we verify that assumptions (i)--(vi) of Theorem~\ref{thm:ttsa-ode} hold for the LKQL recursion.

\noindent\textbf{(i) Step sizes.} Assumption~(i) coincides with Assumption~1 of Theorem~\ref{thm:lkql-value-learn} and therefore holds trivially.

\noindent\textbf{(ii) Lipschitzness.} 
With $\mu\in\set{\Pi}$ fixed, the operators $\op{P}^\mu$, $\op{C}_{n,\lambda}^\mu$, and $\op{A}_\tau^\mu \Let (1-\tau)(\op{I}-\tau\op{P}^\mu)^{-1}$ are constant non-negative linear maps with norms
\begin{equation*}
	\|\op{P}^\mu\|_\infty = \|\op{A}_\tau^\mu\|_\infty = 1,
	\qquad
	\|\op{C}_{n,\lambda}^\mu\|_\infty = \frac{1-(\gamma\lambda)^n}{1-\gamma\lambda} \leq \frac{1}{1-\gamma\lambda}.
\end{equation*}
The Bellman optimality operator $\op{T}$ is a $\gamma$-contraction in $\infty$-norm.
Hence, the Bellman residual operator $\op{T} - \op{I}$ is $(1+\gamma)$-Lipschitz in $\infty$-norm.

\emph{Slow drift $H$.} 
For any $(Q,U),(Q',U')\in\set{Q}\times\set{Q}$, using $\norm{\op{B}_\rho}_\infty\leq 1$, we have
\begin{equation*}
	\norm{H(Q,U)-H(Q',U')}_\infty
	\leq \frac{1+\gamma}{1-\gamma\lambda}\norm{Q-Q'}_\infty + \xi\norm{U-U'}_\infty.
\end{equation*}

\emph{Fast drift $G$.}
Similarly, using $\norm{\op{P}^\mu}_\infty = 1$,
\begin{equation*}
	\norm{G(Q,U)-G(Q',U')}_\infty
	\leq (1-\tau)(1+\gamma)\norm{Q-Q'}_\infty + (1+\tau)\norm{U-U'}_\infty.
\end{equation*}

\emph{Fast equilibrium $U^\star$.}
The fast equilibrium $U^\star(Q) = \op{A}_\tau^\mu(\op{T}Q-Q)$ satisfies the Lipschitz bound
\begin{equation*}
	\norm{U^\star(Q)-U^\star(Q')}_\infty
	\leq \norm{\op{A}_\tau^\mu}_\infty\,(1+\gamma)\norm{Q-Q'}_\infty
	= (1+\gamma)\norm{Q-Q'}_\infty.
\end{equation*}
This verifies (ii) and the Lipschitz hypothesis on $Y^\star(\cdot)=U^\star(\cdot)$ in (v).

\noindent\textbf{(iii) Noise.}
Write the LKQL update~\eqref{eq:lkql-update} in full-vector form
\begin{align*}
	&Q_{k+1} = Q_k + \alpha_k\Delta^Q_{k+1}, \quad \Delta^Q_{k+1}(z_k) \Let \ssum_{j=0}^{n-1}(\gamma\lambda)^j\,\delta^Q_{k,j} + \xi\,U_k(z_k),\\
	&U_{k+1} = U_k + \beta_k\Delta^U_{k+1}, \quad \Delta^U_{k+1}(z_k) \Let \delta^U_{k,0},
\end{align*}
where $\Delta^Q_{k+1}$ and $\Delta^U_{k+1}$ are sparse update vectors supported only at $z_k$, so $\Delta^Q_{k+1}(z) = \Delta^U_{k+1}(z) = 0$ for $z\neq z_k$.
For every $z\in\set{S}\times\set{A}$, conditioning on $\set{F}_k$ and applying the Markov property gives
\begin{align*}
	\E\left[\Delta^Q_{k+1}(z) \mid \set{F}_k\right]
	 & = \rho(z)\left[\ssum_{j=0}^{n-1}(\gamma\lambda)^j
	     \bigl((\op{P}^\mu)^j(\op{T}Q_k-Q_k)\bigr)(z)+\xi U_k(z)\right] \\
	 & = \rho(z)h(Q_k,U_k)(z)
	 = H(Q_k,U_k)(z), \\
	\E\left[\Delta^U_{k+1}(z) \mid \set{F}_k\right]
	 & = \rho(z)\left[(1-\tau)(\op{T}Q_k-Q_k)(z)
	     +\tau(\op{P}^\mu U_k)(z)-U_k(z)\right] \\
	 & = \rho(z)g(Q_k,U_k)(z)
	 = G(Q_k,U_k)(z).
\end{align*}
Consequently, the full-vector conditional expectations satisfy
\begin{equation*}
	\E\left[\Delta^Q_{k+1}\mid\set{F}_k\right]=H(Q_k,U_k),
	\qquad
	\E\left[\Delta^U_{k+1}\mid\set{F}_k\right]=G(Q_k,U_k).
\end{equation*}
We therefore define the martingale-difference noise terms by
\begin{equation*}
	M^Q_{k+1} \Let \Delta^Q_{k+1} - H(Q_k, U_k)
	\;\; \text{and} \;\;
	M^U_{k+1} \Let \Delta^U_{k+1} - G(Q_k, U_k),
\end{equation*}
so the LKQL recursion takes the form of Theorem~\ref{thm:ttsa-ode} with drifts $H=\op{B}_\rho h$ and $G=\op{B}_\rho g$.
By construction, $\E[M^Q_{k+1}\mid\set{F}_k]=\E[M^U_{k+1}\mid\set{F}_k]=0$.

We next bound the noise second moments.
Each TD error satisfies $|\delta^{Q}_{k,j}| \leq \norm{r}_\infty + (1+\gamma)\norm{Q_k}_\infty$.
Hence, for some finite constants $L_1,L_2$, we obtain
\begin{align*}
	\big\|{\Delta^Q_{k+1}}\big\|_2 = \big|{\Delta^Q_{k+1}(z_k)}\big|
	 & \leq \frac{\norm{r}_\infty + (1+\gamma)\norm{Q_k}_\infty}{1-\gamma\lambda} + \xi\norm{U_k}_\infty
	\leq L_1\bigl(1 + \norm{Q_k}_\infty + \norm{U_k}_\infty\bigr),                                          \\
	\big\|{\Delta^U_{k+1}}\big\|_2 =\big|\Delta^U_{k+1}(z_k)\big|
	 & \leq (1-\tau)|\delta^Q_{k,0}| + (1+\tau)\norm{U_k}_\infty \leq L_2(1+\norm{Q_k}_\infty+\norm{U_k}_\infty).
\end{align*}
Combining these bounds with the linear growth of $H$ and $G$ established in (ii), we obtain
\begin{equation*}
	\E\bigl[\big\|M^Q_{k+1}\big\|_2^2 + \big\|M^U_{k+1}\big\|_2^2 \;\big|\; \set{F}_k\bigr]
	\leq K\bigl(1 + \norm{Q_k}_\infty^2 + \norm{U_k}_\infty^2\bigr)
	\leq K\bigl(1 + \norm{Q_k}_2^2 + \norm{U_k}_2^2\bigr),
\end{equation*}
for some $K<\infty$, where the second inequality uses $\norm{\cdot}_\infty \leq \norm{\cdot}_2$. This verifies (iii).

\noindent\textbf{(iv) Boundedness.} To establish boundedness of the iterates, we apply the two-timescale stability theorem of \citet{lakshminarayanan2017stability}.
This theorem guarantees a.s.\ boundedness $\sup_k(\norm{Q_k}_\infty + \norm{U_k}_\infty) < \infty$ under conditions~(i)--(iii) of Theorem~\ref{thm:ttsa-ode} together with the following two additional scaling conditions:
\begin{enumerate}[label=(\alph*)]
	\item the rescaled fast drift $G_c(Q, U) \Let G(c Q, c U)/c$ converges uniformly on compacts as $c \to \infty$ to a limit $G_\infty$, and the ODE $\dot{U}_\infty = G_\infty(Q, U)$ has a unique globally asymptotically stable equilibrium $U^\star_\infty(Q)$ that is Lipschitz in $Q$ with $U^\star_\infty(0) = 0$;
	\item the rescaled slow drift at the fast equilibrium, $H_c(Q) \Let H(c Q,c U^\star_\infty(Q))/c$, converges uniformly on compacts as $c \to \infty$ to a limit $H_\infty$, and the ODE $\dot{Q}_\infty = H_\infty(Q)$ has the origin as its unique globally asymptotically stable equilibrium.
\end{enumerate}
Conditions (i)--(iii) are verified above.
We verify (a) and (b) next.

For $c \geq 1$, scaling $Q$ by $c$ leaves all linear operators in $H$ and $G$ unchanged and only affects the reward term in $\op{T}$. In particular, we have $\frac{\op{T}(c Q)}{c} = \frac{r}{c} + \gamma\op{P}^{\pi_Q}Q$ with $\lim_{c \to \infty} \frac{\op{T}(c Q)}{c} = \gamma\op{P}^{\pi_Q}Q$. 
The map $Q \mapsto \gamma\op{P}^{\pi_Q}Q$ is $\gamma$-Lipschitz in $\infty$-norm and has the unique fixed point $0$. This follows from the same argument as in (ii), applied to the reward-free part of $\op{T}$.

\emph{Scaled fast equilibrium.}
Since $\op{T}(c Q)/c \to \gamma\op{P}^{\pi_Q}Q$ uniformly on compacts as $c \to \infty$, the rescaled fast drift converges uniformly on compacts to
\begin{equation*}
	G_\infty(Q,U) \Let \op{B}_\rho\bigl((1-\tau)(\gamma\op{P}^{\pi_Q}Q - Q) + \tau\op{P}^\mu U - U\bigr).
\end{equation*}
For fixed $Q$, applying Lemma~\ref{lem:fast-gas} with $V = (1-\tau)(\gamma\op{P}^{\pi_Q}Q - Q)$ yields the unique globally asymptotically stable equilibrium $U^\star_\infty(Q) = \op{A}_\tau^\mu(\gamma\op{P}^{\pi_Q}Q - Q)$. 
Observe that $U^\star_\infty(0) = 0$ and that $U^\star_\infty$ is Lipschitz in $Q$ by the same arguments as for $U^\star$ in (ii), with $\op{T}$ replaced by $\gamma\op{P}^{\pi_Q}$.

\emph{Scaled slow ODE.}
Once again, using the uniform convergence $\op{T}(c Q)/c \to \gamma\op{P}^{\pi_Q}Q$ on compacts as $c \to \infty$, and substituting $U^\star_\infty(Q)$ into the rescaled slow drift, we have $H_c(Q) \to H_\infty(Q)$ uniformly on compacts, where
\begin{equation*}
	H_\infty(Q) \Let \op{B}_\rho\bigl(\op{C}_{n,\lambda}^\mu + \xi\op{A}_\tau^\mu\bigr)(\gamma\op{P}^{\pi_Q}Q - Q).
\end{equation*}
Since $\gamma\op{P}^{\pi_Q}$ is $\gamma$-Lipschitz in $\infty$-norm and has fixed point $0$, Lemma~\ref{lem:slow-gas} applied with $\widehat{\op{T}} = \gamma\op{P}^{\pi_Q}$ and $Q^\dagger = 0$ implies that the origin is the unique globally asymptotically stable equilibrium of the scaled slow ODE $\dot{Q}_\infty(t) = H_\infty(Q_\infty(t))$.

Conditions (i)--(iii), (a), and (b) allow us to apply the stability theorem of \citet{lakshminarayanan2017stability}, which gives $\sup_{k\geq 0}\bigl(\norm{Q_k}_\infty + \norm{U_k}_\infty\bigr) < \infty$ almost surely, verifying (iv).

\noindent\textbf{(v) Fast-scale ODE.} For $V \Let (1-\tau)(\op{T}Q-Q)$ and the fast ODE $\dot{U}(t)=\op{B}_\rho\bigl(V + \tau\op{P}^{\mu}U(t) - U(t)\bigr)$,
Lemma~\ref{lem:fast-gas} yields the unique globally asymptotically stable equilibrium $U^\star(Q) = \op{A}_{\tau}^{\mu}(\op{T}Q-Q)$. 
Together with the Lipschitzness of $Q \mapsto U^\star(Q)$ shown in (ii), this verifies (v).

\noindent\textbf{(vi) Slow-scale ODE.} On the slow timescale, the fast iterate is equilibrated at $U^\star(Q)$, giving the limiting slow ODE $\dot{Q}(t)=\op{B}_\rho\bigl(\op{C}_{n,\lambda}^{\mu}+\xi\op{A}_{\tau}^{\mu}\bigr)(\op{T}Q(t)-Q(t))$. 
For $\widehat{\op{T}} = \op{T}$ and $Q^\dagger = Q^\star$, Lemma~\ref{lem:slow-gas} yields $Q^\star$ as the unique globally asymptotically stable equilibrium, verifying (vi).

Finally, combining (i)--(vi), we have $Q_k\to Q^\star$ and $U_k\to U^\star(Q^\star)=0$ almost surely.

\section{Experimental Setup}\label{appx:exp-setup}

\begin{table}[t]
	\centering
	\caption{Shared hyperparameters for the continuous-control experiments. Ranges denote the tuning grid. A single value is fixed across all runs.}
	\label{tab:hparams}
	\begin{tabular}{lll}
		\toprule
		Symbol         & Description                            & Value                                              \\
		\midrule
		$\gamma$       & discount factor                        & $0.99$                                             \\
		$\lambda$      & return trace coefficient               & $0.95$                                             \\
	    $n$            & truncation                             & $[4, \text{rollout length}]$                        \\
	    --             & LK-term truncation                             & $\text{rollout length}$                        \\
		$\lambda_U$    & LK-term trace coefficient              & $[0.5, 0.999]$                                             \\
		$\tau$       & LK-term discount factor                & $[0.5, 0.9999]$                          \\
		$\beta_U$      & LK-term loss coefficient               & $1$ (fixed)                          \\
		--             & LK-term learning rate                  & $[10^{-5}, 10^{-1}]$ (constant)    \\
		$\epsilon_{\mathrm{clip}}$ & PPO clipping parameter                 & $0.2$                                              \\
		$\beta_\pi$    & policy loss coefficient                & $2$                                              \\
		--             & entropy coefficient                    & $0$                                                \\
		--             & actor-critic learning rate             & $3\times10^{-4}$ (linearly annealed)               \\
		--             & gradient-norm clip                     & $0.5$                                              \\
		--             & NAF exploration noise scale            & $\{0.1, 0.2, 0.25\}$                               \\
		$N$            & parallel actors                        & $16$                                               \\
		$T$            & outer iterations                       & $500$                                              \\
		--             & minibatches per epoch                  & $32$                                               \\
		--             & minibatch size                         & $\text{rollout length}/2$                                               \\
		--             & epochs per iteration                   & $10$                                               \\
		$S$            & seeds                                  & $32$                                               \\
		--            & rollout length                         & $\{32, 64, 128, 256, 512\}$                        \\
		\bottomrule
	\end{tabular}
\end{table}

We implement the algorithms in JAX~\cite{jax2018github} and evaluate on \texttt{Hopper-v5}, \texttt{Walker2d-v5}, and \texttt{HalfCheetah-v5} from MuJoCo~\cite{todorov2012mujoco}, simulated with its JAX-based MJX engine.
LKQL and its baselines share the same actor-critic backbone and sampling loop (Algorithm~\ref{alg:lkql}), where the baselines are recovered exactly by disabling the LK term ($U_\omega \equiv 0$), yielding GAE in the on-policy setting and THQL in the off-policy setting.
We use the \texttt{rlax} package~\cite{deepmind2020jax} to compute the truncated returns for both the values and the LK term.

\textbf{Architectures.}
All networks are multilayer perceptrons with $\tanh$ activations, layer normalization, and orthogonal initialization.
The on-policy backbone parameterizes the policy $\pi_\theta$ with hidden sizes $(64, 64)$ and a state-independent log standard deviation, and the value $V_\phi$ with $(128, 128)$.
The off-policy NAF backbone parameterizes the state value $V_\phi$ with $(128, 128)$, the deterministic mean $\bar\pi_\theta$ with $(64, 64)$, and the (diagonal) precision $P_\phi$ with $(128, 128)$.
In both estimators, the LK term $U_\omega$ is a separate network with the same hidden sizes as the critic, $(128, 128)$, mapping the observation to a scalar.

\textbf{Optimization.}
We optimize with Adam ($\epsilon_{\mathrm{Adam}} = 10^{-5}$) and clip gradients to a global norm of $0.5$.
The actor and critic share a single optimizer with learning rate $3\times10^{-4}$ annealed linearly to zero over training, whereas the LK term $U_\omega$ uses a separate optimizer with a constant learning rate, a choice motivated by the two-timescale analysis of Theorem~\ref{thm:lkql-value-learn}.
Each run uses $N = 16$ parallel actors and $T = 500$ outer iterations.
Each iteration collects a rollout per actor and performs $10$ epochs of updates with $32$ minibatches per epoch.
We report the final return of the deterministic policy, averaged over $S = 32$ seeds.
The rollout length is varied over the set $\{32, 64, 128, 256, 512\}$.
Since each iteration collects $(\text{rollout length} \times N)$ transitions, the total number of environment steps scales with the rollout length.

\textbf{Hyperparameters.}
Table~\ref{tab:hparams} lists the shared settings.
The baselines are tuned over the rollout length and, for NAF, the exploration noise scale. LKQL inherits the resulting best settings and tunes only its own hyperparameters, namely, the LK trace coefficient $\lambda_U$, the LK-term discount factor $\tau$, the $n$-step truncation horizon, and the optimizer parameters for $\omega$.

\bibliographystyle{apalike}
\begin{small}
\bibliography{references}
\end{small}

\end{document}

%% file: arxiv/AMIN_style.tex
\makeatletter
\def\@settitle{\begin{center}%
		\baselineskip14\p@\relax
		\normalfont\LARGE\bfseries
		\@title
	\end{center}%
}

\def\section{\@startsection{section}{1}%
	\z@{.7\linespacing\@plus\linespacing}{.5\linespacing}%
	{\normalfont\large\bfseries}}

\def\subsection{\@startsection{subsection}{2}%
	\z@{.5\linespacing\@plus.7\linespacing}{.5\linespacing}%
	{\normalfont\bfseries}}

\def\@setauthors{%
  \begingroup
  \def\thanks{\protect\thanks@warning}%
  \trivlist
  \centering\footnotesize \@topsep30\p@\relax
  \advance\@topsep by -\baselineskip
  \item\relax
  \author@andify\authors
  \def\\{\protect\linebreak}%
  \authors%
  \ifx\@empty\contribs
  \else
    ,\penalty-3 \space \@setcontribs
    \@closetoccontribs
  \fi
  \endtrivlist
  \endgroup
}

\makeatother

\usepackage[table, xcdraw, usenames, dvipsnames]{xcolor}
\definecolor{darkblue}{rgb}{0.0, 0.0, 0.45}
\definecolor{darkgreen}{rgb}{0.0, 0.45, 0}
\usepackage[colorlinks	= true,
raiselinks	= true,
linkcolor	= darkblue, 
citecolor	= Mahogany,
urlcolor	= darkgreen,
pdfauthor	= {},
pdftitle	= {},
pdfkeywords	= {},
pdfsubject	= {},
plainpages	= false]{hyperref}

\pdfoutput=1
\date{\today}

%% file: arxiv/notation.tex
\newcommand{\set}[1]{\bm{\mathrm{#1}}}   
\newcommand{\op}[1]{\mathcal{#1}}   
\newcommand{\prob}{\mathds{P}}

\newcommand{\e}{\bm{1}}
\DeclareMathOperator*{\argmax}{arg\,max}

\newcommand{\wt}{\widetilde}
\newcommand{\norm}[1]{\left\Vert #1 \right\Vert}

\newcommand{\wh}{\widehat}
\newcommand{\Let}{\triangleq}
\newcommand{\E}{\mathds{E}}
\newcommand{\R}{\mathbb{R}}
\def\ssum{\begingroup\textstyle \sum\endgroup}

\theoremstyle{plain}
\newtheorem{Thm}{Theorem}[section]

\newtheorem{Lem}[Thm]{Lemma}
\newtheorem{Cor}[Thm]{Corollary}

\newtheorem{Rem}[Thm]{Remark}


%% file: figure/lk_operator_diagram.tex
\begin{figure}[t]
\centering
\begin{tikzpicture}[
    state/.style={circle, draw, minimum size=3mm, inner sep=0pt},
    action/.style={circle, fill=black, minimum size=1.8mm, inner sep=0pt},
    side/.style={circle, fill=black, minimum size=1.8mm, inner sep=0pt},
    mixed/.style={circle, draw, thick, double, minimum size=5mm, inner sep=0pt, font=\scriptsize},
    arr/.style={-{latex}, semithick, shorten >=0.5mm, shorten <=0.5mm},
    sarr/.style={-{latex}, thin, shorten >=0.5mm, shorten <=0.5mm},
]

\node[state, ] (s0) at (0, 0) {};
\node[action, right=10mm of s0] (a0) {};
\node[state, right=10mm of a0] (s1) {};
\node[action, right=10mm of s1] (a1) {};
\node[right=10mm of a1, font=\normalsize] (dots) {$\cdots$};
\node[state, right=10mm of dots] (sn) {};
\node[action, right=10mm of sn] (an) {};
\node[mixed, right=10mm of an] (m) {$\infty$};

\draw[arr] (s0) -- (a0);
\draw[arr] (a0) -- node[above, font=\scriptsize] {$\gamma\lambda$} (s1);
\draw[arr] (s1) -- (a1);
\draw[arr] (a1) -- node[above, font=\scriptsize] {$\gamma\lambda$} (dots);
\draw[arr] (dots) -- (sn);
\draw[arr] (sn) -- (an);
\draw[arr] (an) -- node[above, font=\scriptsize] {$\gamma\lambda$} (m);

\node[side, above=3mm of a1] (a1u) {};
\node[side, below=3mm of a1] (a1d) {};
\draw[sarr] (s1) -- (a1u);
\draw[sarr] (s1) -- (a1d);

\node[side, above=3mm of an] (anu) {};
\node[side, below=3mm of an] (and) {};
\draw[sarr] (sn) -- (anu);
\draw[sarr] (sn) -- (and);

\begin{scope}[on background layer]
\fill[gray!15, rounded corners=3pt] ($(an.east)+(2mm,-6mm)$) rectangle ($(m.east)+(16mm,6mm)$);
\end{scope}

\draw[arr] (m) to[out=50, in=-10, looseness=4] node[right, font=\scriptsize] {$\gamma\lambda\op{P}_\infty^\mu$} (m);

\end{tikzpicture}
\caption{Backup diagram of the LK Harutyunyan operator.
Hollow circles denote states, and filled circles denote actions.
The first $n$ steps follow the chain induced by $\mu$.
The shaded region depicts tail completion with the LK, and the omitted tail is replaced by stationary propagation from step $n$ onward.
Removing the LK term recovers the $n$-step Harutyunyan operator.}
\label{fig:backup}
\end{figure}